\documentclass{article} 
\usepackage{iclr2027_conference,times}

\usepackage{amsmath,amsfonts,bm}

\def\eqref#1{equation~\ref{#1}}

\def\1{\bm{1}}

\DeclareMathAlphabet{\mathsfit}{\encodingdefault}{\sfdefault}{m}{sl}
\SetMathAlphabet{\mathsfit}{bold}{\encodingdefault}{\sfdefault}{bx}{n}

\iclrfinalcopy  
\usepackage{hyperref}
\usepackage{url}

\usepackage[utf8]{inputenc} 
\usepackage[T1]{fontenc}    
\usepackage{hyperref}       
\usepackage{url}            
\usepackage{booktabs}       
\usepackage{amsfonts}       
\usepackage{nicefrac}       
\usepackage{microtype}      
\usepackage{xcolor}         

\usepackage{amsmath}
\usepackage{amssymb}
\usepackage{mathtools}
\usepackage{amsthm}
\usepackage{comment}
\usepackage{graphicx}
\usepackage{subcaption}

\theoremstyle{plain}
\newtheorem{theorem}{Theorem}[section]

\theoremstyle{definition}

\newtheorem{assumption}[theorem]{Assumption}
\theoremstyle{remark}

\usepackage{amsthm}

\usepackage{multirow}

\usepackage{titletoc}
\usepackage{xcolor}         

\usepackage{tabularx}
\usepackage{graphicx}
\usepackage{colortbl}

\definecolor{LightMint}{RGB}{170,235,205}

\definecolor{LightRed}{RGB}{255,182,193} 
\definecolor{LightBlue}{RGB}{173, 216, 230}
\definecolor{LightGreen}{RGB}{144, 238, 144}

\definecolor{LightMint}{RGB}{213,245,230}
\definecolor{LightRed}{RGB}{255,219,224}
\definecolor{LightBlue}{RGB}{214,236,243}
\definecolor{LightGreen}{RGB}{200,247,200}

\colorlet{shadecolor}{LightGreen}
\colorlet{shadecolor}{LightBlue}
\colorlet{shadecolor}{LightRed}

\usepackage{enumitem}
\setlist[enumerate]{leftmargin=3em}

\usepackage{arydshln} 
\usepackage{algorithm}
\usepackage{algpseudocode}

\usepackage{graphicx}
\usepackage{wrapfig}

\title{FedMIX-P: Mixing Local and Global Preconditioners for Federated Vision and Language Model Training}

\author{ Junkang Liu\textsuperscript{1}\\
	Tianjin University\\
	\texttt{\{junkangliukk\}@gmail.com}
}

\begin{document}

	\maketitle
	\pagestyle{plain}
	\thispagestyle{plain}
	
\begin{abstract}
Adaptive preconditioners improve the efficiency of vision and language model training, but using them independently on federated clients can compromise global optimization. Under heterogeneous data, clients develop different preconditioners, causing their transformed gradients to aggregate into a biased direction even when evaluated at the same model. Preconditioners therefore affect both local training efficiency and the consistency of global updates. Synchronizing them at the beginning of each round does not fully resolve this problem, as their states can diverge during local training.
We propose \texttt{FedMIX-P}, which mixes a shared reference with each client's adaptive preconditioner at every local step. This design retains local adaptation while controlling the differences between the operators applied to gradients. The mean-squared operator mismatch is exactly $\lambda^2$ times the local-state drift. For smooth nonconvex objectives with stochastic gradients and partial client participation, we establish stationarity bounds under bounded preconditioner drift. Suitable stepsizes and a horizon-dependent mixing weight yield an $O(R^{-1/2})$ bound on the average expected squared global-gradient norm, without requiring local preconditioners to converge to one another. A two-client example shows that a fixed positive mixing weight can preserve a nonstationary fixed point. The analysis covers bounded linear symmetric positive-definite preconditioners.
Experiments evaluate SOAP, Sophia, and Muon variants on image classification, vision-model fine-tuning, and federated language-model pretraining. \texttt{FedMIX-P} improves over the corresponding local optimizers across the evaluated settings, with accuracy gains of up to 19.47 percentage points and lower validation loss for language models ranging from 60M to 350M parameters. These results demonstrate the practical value of coordinating preconditioners in federated model training, while the full nonlinear and momentum-based updates require separate theoretical analysis.
\end{abstract}
\section{Introduction}
\label{sec:intro}
Effective local optimization is central to federated training of vision and language models. Each communication round incurs the cost of exchanging model parameters, so the progress obtained from local computation matters. Preconditioners offer a way to improve this progress by adapting gradient updates to parameter geometry. SOAP~\cite{vyas2025soap} combines matrix structure with adaptive scaling, Sophia~\cite{liu2024sophia} uses diagonal curvature estimates, and Muon~\cite{jordan6muon} orthogonalizes matrix update directions. These methods illustrate the value of adapting update geometry in model training and motivate bringing their benefits to federated learning (FL).

Much of federated optimization, however, has been developed around local SGD and the aggregation of its updates, following FedAvg~\cite{mcmahan2017communication}. This perspective emphasizes how heterogeneous data cause local models to drift apart. With adaptive preconditioning, another source of disagreement arises: clients can transform gradients differently even at the same model. A preconditioner therefore determines not only the effectiveness of a client's local steps, but also how those steps contribute to the global update. Understanding this interaction is necessary to turn the benefits of local preconditioning into better federated training.

A direct approach equips each client with a local preconditioned optimizer and averages the resulting model changes. We refer to this family as \texttt{FedPO}. Our experiments show that local adaptive optimizers can substantially outperform local SGD in federated language-model training, demonstrating the value of stronger local optimization. Yet their effectiveness can deteriorate under heterogeneous data. As optimizer states adapt to different client gradient statistics, the resulting transformations can bias the aggregated direction. This effect is distinct from local model drift and cannot be explained by differences in gradient-query locations alone.
The central question is whether the progress achieved by stronger local optimizers survives aggregation. Equation~\eqref{eq:intro_geometry_coupling} identifies one obstruction: client-specific preconditioners interact with gradient heterogeneity to alter the global direction, even before local trajectories separate.

To separate preconditioner mismatch from local model drift, consider one step at a shared model $x$. Let $g_i=\nabla f_i(x)$, $g=\frac{1}{N}\sum_i g_i$, and $\bar P=\frac{1}{N}\sum_i P_i$, where $P_i$ is client $i$'s linear preconditioner. The aggregated direction admits the decomposition
\begin{equation}
	\frac1N\sum_{i=1}^N P_i g_i
	=
	\bar P g+\mathcal C(x),
	\qquad
	\mathcal C(x)
	:=
	\frac1N\sum_{i=1}^N
	(P_i-\bar P)(g_i-g).
	\label{eq:intro_geometry_coupling}
\end{equation}
Here, $\mathcal C(x)$ captures the coupling between preconditioner mismatch and gradient heterogeneity. It measures the departure of the aggregated direction from the global gradient transformed by the average preconditioner. At a global stationary point, $g=0$, this coupling can still produce a nonzero update. Thus, differences in client geometry can bias aggregation even before local models diverge. Round-start synchronization removes the initial preconditioner mismatch, but subsequent adaptation to local data can recreate it within the same round.

We propose \texttt{FedMIX-P}, which mixes local and global preconditioners at every client step. At the beginning of round $r$, the server broadcasts a shared reference $P^r$, and participating clients initialize their local preconditioners from it. Client $i$ then applies
\begin{equation}
	H_i^{r,k}=(1-\lambda_r)P^r+\lambda_rP_i^{r,k},
	\qquad
	x_i^{r,k+1}
	=
	x_i^{r,k}-\eta_{l,r}H_i^{r,k}g_i^{r,k}.
	\label{eq:intro_mix}
\end{equation}
The server aggregates the returned model changes and terminal preconditioners to update the model and shared reference. The local component adapts to client data, while the shared component coordinates the geometry used throughout the round. Their mixture gives direct control over the applied operator mismatch:
$
	\|H_i^{r,k}-P^r\|_{\mathrm{op}}^2
	=
	\lambda_r^2
	\|P_i^{r,k}-P^r\|_{\mathrm{op}}^2.
	\label{eq:intro_effective_drift}
$
Thus, local preconditioners can retain distinct information while their influence on aggregation bias is controlled through the mixing weight.

For smooth nonconvex objectives, we analyze stochastic gradients and uniform partial client participation under bounded heterogeneity and bounded linear symmetric positive-definite preconditioners. The stationarity bound separates optimization error, sampling variation, local trajectory drift, and the contribution of preconditioner mismatch. The last contribution scales with $\lambda^2\bar D_P$, where $\bar D_P$ bounds local preconditioner drift. For an $R$-round run, choosing constant local stepsize $\eta_l=\Theta(1/(K\sqrt R))$, server stepsize $\gamma=K\eta_l$, and a horizon-dependent mixing weight $\lambda=\Theta(R^{-1/2})$ yields an $O(R^{-1/2})$ bound on the average expected squared global-gradient norm. This result does not require local preconditioners to converge to one another. The analysis also accounts for operators constructed from the current mini-batch, whose transformed gradient noise can be biased. A two-client example shows that fixed local preconditioner influence can preserve a nonstationary fixed point.

Experiments evaluate practical SOAP, Sophia, and Muon variants on image classification, pretrained vision-model fine-tuning, and federated language-model pretraining. The results show that coordinating local and shared geometry improves on the corresponding local optimizers, with particularly large gains under severe data heterogeneity. Together, the analysis and experiments support preconditioner coordination as a way to make stronger local optimization more effective in federated training. The theoretical guarantees cover the stated linear preconditioner model; the complete updates with nonlinear orthogonalization, clipping, or momentum require separate analysis.

Our contributions are as follows:
\begin{itemize}
	\item \textbf{A mechanism for biased aggregation.}
	We identify \emph{geometry--heterogeneity coupling}: even when clients
	evaluate gradients at the same model, differences in their
	preconditioners interact with differences in their data distributions
	to bias the aggregated direction. The exact decomposition
	in~\eqref{eq:intro_geometry_coupling} separates this effect from
	trajectory drift caused by multiple local steps.
	
	\item \textbf{Direct control of the applied geometry.}
	We propose \texttt{FedMIX-P}, which combines a shared preconditioner
	with each client's adaptive preconditioner at every local step.
	This design controls the mismatch between the \emph{operators
		actually used for updates} and the shared reference, even when
	local preconditioner states do not converge to one another.
	
	\item \textbf{A drift-aware convergence characterization.}
	For stochastic non-convex optimization with partial client
	participation, we derive stationarity bounds that account for
	preconditioners dependent on the current gradient sample.
	The \texttt{FedMIX-P} drift term scales with
	$\lambda^2\bar D_P$: a decaying mixing weight yields a vanishing
	bound under bounded state drift. We contrast this result with
	FedPO's drift-dependent bound and give a counterexample showing
	that fixed local mixing can converge to a nonstationary point.
\end{itemize}

\section{Related Work}
\label{sec:related}
\vspace{-2mm}

\textbf{Heterogeneity-aware federated optimization.}
FedAvg~\citep{mcmahan2017communication} combines multiple local
updates through model averaging, but heterogeneous client data can
cause local trajectories to deviate from the global descent
direction. FedProx~\citep{li2018federated} constrains local
optimization with a proximal term; SCAFFOLD~\citep{karimireddy2020scaffold}
uses server--client control variates to correct client drift;
and FedCM~\citep{xu2021fedcm} incorporates client-level momentum.
These approaches primarily regulate local model trajectories or
gradient-based update directions. When clients additionally apply
different adaptive preconditioners, an aggregation mismatch can
already arise at a shared model:
$\frac1N\sum_iP_i\nabla f_i(x)$ need not equal
$\bar P\nabla f(x)$. \texttt{FedMIX-P} addresses this operator--gradient
coupling by controlling the preconditioners applied during local
training.

\textbf{Structure-aware optimizers.}
Scalable alternatives to full Newton updates capture different
forms of parameter geometry. Shampoo~\citep{gupta2018shampoo}
maintains structured preconditioning matrices along tensor
dimensions. SOAP~\citep{vyas2025soap} combines Shampoo's evolving
eigenbasis with Adam-style updates in that basis.
Sophia~\citep{liu2024sophia} uses lightweight diagonal Hessian
estimates and clipping. Muon~\citep{jordan6muon} approximately
orthogonalizes matrix momentum using Newton--Schulz iterations;
its nonlinear update should not be identified with a fixed linear
positive-definite preconditioner. These methods motivate the use
of adaptive geometry on clients, while their distinct state and
update rules also delimit which implementations are covered by a
unified linear-preconditioner analysis.

\textbf{Preconditioning and matrix optimization in FL.}
Fed-Sophia~\citep{elbakary2024fed} brings Sophia-style curvature
information into federated optimization. FedPM~\citep{ishii2026fedpm}
explicitly studies local preconditioner drift and uses
\emph{preconditioned mixing of local parameters at the server};
its theoretical result focuses on a strongly convex, single-local-step
setting. Federated Muon methods~\citep{liu2025fedmuon,
	takezawa2025fedmuon} investigate matrix orthogonalization and
the bias of locally aggregated Muon directions. These studies
establish that client-specific update geometry is a substantive
federated optimization challenge.

\texttt{FedMIX-P} intervenes at a different point in the update: it mixes
the shared and local \emph{preconditioners at every client step},
$H_i^{r,k}=(1-\lambda_r)P^r+\lambda_rP_i^{r,k}$.
This yields the exact effective-drift relation
$\|H_i^{r,k}-P^r\|_{\mathrm{op}}^2
=\lambda_r^2\|P_i^{r,k}-P^r\|_{\mathrm{op}}^2$.
Our nonconvex analysis covers multiple local steps, stochastic
gradients, and partial participation, and obtains a vanishing
stationarity bound under bounded local-state drift by decreasing
$\lambda_r$. This guarantee is stated for linear positive-definite
operators; it does not automatically extend to Muon's nonlinear
orthogonalization or the complete SOAP and Sophia update rules.


\section{The Proposed Unified Algorithm}
\vspace{-2mm}
\label{sec:method}
\subsection{Problem Setup}
\vspace{-2mm}
FL seeks to learn a global model collaboratively over clients by minimizing the population risk:
\vspace{-2mm}
\begin{equation}
	f_i(\boldsymbol{x}) 
	:= \mathbb{E}_{\xi_i \sim \mathcal{D}_i}
	\bigl[F_i(\boldsymbol{x}; \xi_i)\bigr], \qquad f(\boldsymbol{x}) := \frac{1}{N} \sum_{i=1}^N f_i(\boldsymbol{x}).
	\label{eq:objective}
\end{equation}
The function $F_i$ is the loss function on client $i$. $\mathbb{E}_{\xi_i \sim \mathcal{D}_i}[\cdot]$ denotes conditional expectation with respect to  the sample $\xi_i$. $N$ is the number of clients, and 
$\boldsymbol{x}$ is global model.

\begin{algorithm}[tb]
	\small
	\caption{Federated Preconditioned Optimization (FedPO)}
	\label{algorithm_FedPO}
	\begin{algorithmic}[1]
		\Require Local stepsize $\eta_l$, communication rounds $R$,
		local steps $K$, number of clients $N$
		\State Initialize global model $\boldsymbol{x}^0$
		and local preconditioner states $\{P_i^{0,0}\}_{i=1}^{N}$
		\For{$r=0,\ldots,R-1$}
		\State Server samples $\mathcal S_r$ and broadcasts
		$\boldsymbol{x}^r$
		\For{each client $i\in\mathcal S_r$ in parallel}
		\State $\boldsymbol{x}_i^{r,0}\gets\boldsymbol{x}^r$;
		restore client $i$'s retained state $P_i^{r,0}$
		\For{$k=0,\ldots,K-1$}
		\State Sample $B_i^{r,k}$;
		$\boldsymbol{g}_i^{r,k}\gets
		\nabla F_i(\boldsymbol{x}_i^{r,k};B_i^{r,k})$
		\State $\tilde{\boldsymbol{g}}_i^{r,k}
		\gets P_i^{r,k}\boldsymbol{g}_i^{r,k}$
		\State $\boldsymbol{x}_i^{r,k+1}
		\gets\boldsymbol{x}_i^{r,k}
		-\eta_l\tilde{\boldsymbol{g}}_i^{r,k}$
		\State $P_i^{r,k+1}\gets
		\operatorname{UpdatePreconditioner}
		(P_i^{r,k},\boldsymbol{g}_i^{r,k})$
		\EndFor
		\State Retain $P_i^{r,K}$ locally
		\State Send $\Delta\boldsymbol{x}_i^r
		\gets\boldsymbol{x}_i^{r,K}-\boldsymbol{x}^r$
		\EndFor
		\State $\boldsymbol{x}^{r+1}\gets\boldsymbol{x}^r+
		\dfrac{1}{|\mathcal S_r|}
		\displaystyle\sum_{i\in\mathcal S_r}
		\Delta\boldsymbol{x}_i^r$
		\EndFor
	\end{algorithmic}
\end{algorithm}

\subsection{Federated Preconditioned Optimization}
\label{subsec:fedpo}
\vspace{-2mm}

We first formalize the direct use of client-side preconditioned
optimizers in federated learning. We call this baseline
\emph{Federated Preconditioned Optimization} (\texttt{FedPO});
its procedure is given in Algorithm~\ref{algorithm_FedPO}.
This terminology covers both linear preconditioners and more
general optimizer-specific gradient transformations, without
requiring every instance to be a second-order method.

At local step $k$ of round $r$, client $i$ computes
$
	\boldsymbol{g}_i^{r,k}
	=
	\nabla F_i(\boldsymbol{x}_i^{r,k};B_i^{r,k}),
$
where $B_i^{r,k}$ is a sampled mini-batch. The client maintains
an optimizer state $P_i^{r,k}$ and applies the corresponding
gradient transformation:
\begin{equation}
	\label{eq:local-preconditioned-update}
	\boldsymbol{x}_i^{r,k+1}
	=
	\boldsymbol{x}_i^{r,k}
	-
	\eta_l P_i^{r,k}
	\bigl(\boldsymbol{g}_i^{r,k}\bigr).
\end{equation}
The applied state $P_i^{r,k}$ is the state available at the
beginning of step $k$. After the model update, the client
updates it for the next step:
\begin{equation}
	P_i^{r,k+1}
	=
	\operatorname{UpdatePreconditioner}
	\bigl(P_i^{r,k},\boldsymbol{g}_i^{r,k}\bigr).
	\label{eq:fedpo-state-update}
\end{equation}
For example, SOAP uses matrix-structured statistics, Sophia
uses curvature-informed coordinate-wise scaling, and Muon
orthogonalizes matrix update directions. These transformations
need not all be linear; the convergence analysis below
explicitly identifies the symmetric positive-definite linear
subclass to which its guarantees apply.

After $K$ steps, client $i$ sends only its model difference
$
	\Delta\boldsymbol{x}_i^r
	=
	\boldsymbol{x}_i^{r,K}-\boldsymbol{x}^{r}
$
to the server. FedPO aggregates these differences as
$
	\boldsymbol{x}^{r+1}
	=
	\boldsymbol{x}^{r}
	+
	\frac{1}{|\mathcal S_r|}
	\sum_{i\in\mathcal S_r}
	\Delta\boldsymbol{x}_i^r.
	\label{eq:fedpo-aggregation}
$
Each client retains its optimizer state for subsequent
participation; FedPO does not synchronize preconditioners.
The resulting global direction therefore combines gradients
transformed by potentially different client geometries.
The next subsection isolates the bias caused by this mismatch.

\subsection{Preconditioner Drift and Aggregation Bias}
\label{subsec:preconditioner-drift}

To separate geometric mismatch from the usual divergence of
local model trajectories, first consider a single update at
a shared model $\boldsymbol{x}$. Let
$\boldsymbol{g}_i=\nabla f_i(\boldsymbol{x})$,
$\boldsymbol{g}=\nabla f(\boldsymbol{x})$, and suppose that
the applied preconditioners are linear operators
$P_1,\ldots,P_N$. With
$\bar P=\frac{1}{N}\sum_{i=1}^{N}P_i$, their aggregated direction
admits the exact decomposition
\begin{equation}
	\label{eq:geometry-heterogeneity-coupling}
	\frac{1}{N}\sum_{i=1}^{N}P_i\boldsymbol{g}_i
	=
	\bar P\boldsymbol{g}
	+
	\frac{1}{N}\sum_{i=1}^{N}
	(P_i-\bar P)(\boldsymbol{g}_i-\boldsymbol{g}).
\end{equation}
The second term couples preconditioner mismatch with gradient
heterogeneity. It need not vanish when
$\nabla f(\boldsymbol{x})=\boldsymbol{0}$: independently
preconditioned client gradients may still produce a nonzero
aggregate direction at a stationary point of the global
objective. Multiple local steps introduce an additional
trajectory error because the gradients are then evaluated
at different client models.

In \texttt{FedMIX-P}, the server broadcasts a shared reference
$P^r$, and each participating client starts the round from
$P_i^{r,0}=P^r$. We define the round-wise drift of the
\emph{local preconditioner states} relative to this reference
as
\begin{equation}
	\label{eq:state-preconditioner-drift}
	D_{\mathrm{state}}^r
	:=
	\frac{1}{|\mathcal S_r|K}
	\sum_{i\in\mathcal S_r}
	\sum_{k=0}^{K-1}
	\bigl\|P_i^{r,k}-P^r\bigr\|_{\mathrm{op}}^2.
\end{equation}
The index range $k=0,\ldots,K-1$ covers exactly the states
applied to gradients. The uploaded terminal state
$P_i^{r,K}$ is used to construct the next shared reference;
it is not applied during round $r$. If a drift assumption
instead sums over $k=1,\ldots,K$, round-start alignment
$P_i^{r,0}=P^r$ implies that it also upper-bounds
\eqref{eq:state-preconditioner-drift}; the index-transfer
identity is given in the appendix.

\texttt{FedMIX-P} applies the mixed operator
\begin{equation}
	H_i^{r,k}
	=
	(1-\lambda_r)P^r+\lambda_rP_i^{r,k},
	\qquad 0\leq\lambda_r\leq1.
	\label{eq:effective-preconditioner}
\end{equation}
Accordingly, its \emph{effective} drift is
\begin{equation}
	\label{eq:effective-preconditioner-drift}
	D_{\mathrm{eff}}^r
	:=
	\frac{1}{|\mathcal S_r|K}
	\sum_{i\in\mathcal S_r}
	\sum_{k=0}^{K-1}
	\bigl\|H_i^{r,k}-P^r\bigr\|_{\mathrm{op}}^2
	=
	\lambda_r^2D_{\mathrm{state}}^r.
\end{equation}
This equality captures the mechanism of \texttt{FedMIX-P}. Resetting
states at the start of a round removes the initial mismatch,
but local states can separate again as clients process
heterogeneous data. Stepwise mixing limits the mismatch of
the operators that actually transform gradients. In
particular, a bounded $D_{\mathrm{state}}^r$ does not have to
vanish for $D_{\mathrm{eff}}^r$ to vanish as
$\lambda_r\rightarrow0$.


\begin{algorithm}[tb]
	\small
	\caption{\texttt{FedMIX-P}: Federated Mixing of Local and Global Preconditioners}
	\label{algorithm_FedPACT}
	\begin{algorithmic}[1]
		\Require Local stepsizes $\{\eta_{l,r}\}_{r=0}^{R-1}$,
		server stepsizes $\{\gamma_r\}_{r=0}^{R-1}$,
		local-preconditioner weights $\{\lambda_r\}_{r=0}^{R-1}$,
		rounds $R$, local steps $K$, number of clients $N$
		\State Initialize global model $\boldsymbol{x}^0$
		and shared preconditioner $P^0$
		\For{$r=0,\dots,R-1$}
		\State Server samples clients $\mathcal{S}_r$
		and broadcasts $(\boldsymbol{x}^r,P^r)$
		\For{each client $i\in\mathcal{S}_r$ in parallel}
		\State $\boldsymbol{x}_i^{r,0}\gets\boldsymbol{x}^r$;
		\fcolorbox{LightRed}{LightRed}
		{$P_i^{r,0}\gets P^r$}
		\For{$k=0,\dots,K-1$}
		\State Sample mini-batch $B_i^{r,k}$;
		$\boldsymbol{g}_i^{r,k}
		\gets\nabla F_i(\boldsymbol{x}_i^{r,k};B_i^{r,k})$
		\State \fcolorbox{LightBlue}{LightBlue}
		{$H_i^{r,k}\gets
			(1-\lambda_r)P^r+\lambda_r P_i^{r,k}$}
		\State \fcolorbox{LightBlue}{LightBlue}
		{$\boldsymbol{x}_i^{r,k+1}
			\gets\boldsymbol{x}_i^{r,k}
			-\eta_{l,r}H_i^{r,k}
			\boldsymbol{g}_i^{r,k}$}
		\State $P_i^{r,k+1}\gets
		\operatorname{Update\_Preconditioner}
		(P_i^{r,k},\boldsymbol{g}_i^{r,k})$
		\EndFor
		\State Send
		$\bigl(\Delta\boldsymbol{x}_i^r,
		P_i^{r,K}\bigr)$, where
		$\Delta\boldsymbol{x}_i^r
		\gets\boldsymbol{x}_i^{r,K}-\boldsymbol{x}^r$
		\EndFor
		\State $\boldsymbol{d}^{r+1}
		\gets-\dfrac{1}{|\mathcal{S}_r|K\eta_{l,r}}
		\displaystyle\sum_{i\in\mathcal{S}_r}
		\Delta\boldsymbol{x}_i^r$
		\State $\boldsymbol{x}^{r+1}
		\gets\boldsymbol{x}^r-\gamma_r\boldsymbol{d}^{r+1}$
		\State \fcolorbox{LightRed}{LightRed}
		{$P^{r+1}\gets
			\dfrac{1}{|\mathcal{S}_r|}
			\displaystyle\sum_{i\in\mathcal{S}_r}P_i^{r,K}$}
		\EndFor
	\end{algorithmic}
\end{algorithm}

\section{Federated Mixing of Local and Global Preconditioners}
\label{sec:fedmixp}
\vspace{-2mm}

Equation~\eqref{eq:geometry-heterogeneity-coupling} shows that
heterogeneous client preconditioners can introduce a systematic
aggregation bias even when all clients evaluate gradients at
the same model. Synchronizing optimizer states only at the
beginning of a round removes the initial mismatch, but the
states may diverge again during local training. This suggests
a more direct design principle: control the geometry
\emph{applied to every local gradient}, rather than relying
only on the geometry used to initialize a round.

We propose \texttt{FedMIX-P} (\textbf{Fed}erated
\textbf{Mix}ing of Local and Global \textbf{P}reconditioners),
shown in Algorithm~\ref{algorithm_FedPACT}. It combines a
shared preconditioner reference with locally updated
preconditioners through two coupled operations.

\textbf{Shared reference and round-start alignment.}
The server maintains a shared preconditioner $P^r$ and
broadcasts $(\boldsymbol{x}^r,P^r)$ to the selected clients.
Each client initializes its local model and preconditioner
state as
\begin{equation}
	\boldsymbol{x}_i^{r,0}=\boldsymbol{x}^r,
	\qquad
	P_i^{r,0}=P^r.
	\label{eq:fedmix-initialization}
\end{equation}
After $K$ local steps, client $i$ uploads its model
difference $\Delta\boldsymbol{x}_i^r
=\boldsymbol{x}_i^{r,K}-\boldsymbol{x}^r$ and terminal
preconditioner state $P_i^{r,K}$. The server forms the next
reference by averaging the uploaded states:
\begin{equation}
	P^{r+1}
	=
	\frac{1}{|\mathcal S_r|}
	\sum_{i\in\mathcal S_r}P_i^{r,K}.
	\label{eq:fedmix-reference-update}
\end{equation}
Thus every participating client starts a round from the
same reference, while its state can still adapt to local
data within the round.

\textbf{Stepwise geometry mixing.}
At local step $k$, client $i$ constructs the operator used
for its current gradient as
\begin{equation}
	H_i^{r,k}
	=
	(1-\lambda_r)P^r+\lambda_rP_i^{r,k},
	\qquad 0\leq\lambda_r\leq1,
	\label{eq:fedmix-operator}
\end{equation}
and updates
\begin{equation}
	\boldsymbol{x}_i^{r,k+1}
	=
	\boldsymbol{x}_i^{r,k}
	-
	\eta_{l,r}H_i^{r,k}\boldsymbol{g}_i^{r,k}.
	\label{eq:fedmix-local-update}
\end{equation}
The current gradient subsequently updates the local
preconditioner state for step $k+1$. Hence the operator
$P_i^{r,k}$ in~\eqref{eq:fedmix-local-update} is available
before the current mini-batch is sampled. This ordering
also makes the stochastic direction conditionally unbiased
with respect to that mini-batch in the predictable-operator
case analyzed in the appendix.

The mixing weight has a precise geometric meaning:
\begin{equation}
	H_i^{r,k}-P^r
	=
	\lambda_r(P_i^{r,k}-P^r).
	\label{eq:fedmix-contraction-identity}
\end{equation}
Consequently, if $D_{\mathrm{state}}^r$ denotes the
mean-squared drift of the local states applied within
round $r$, then the drift of the operators that actually
transform gradients is exactly
\begin{equation}
	D_{\mathrm{eff}}^r
	=
	\lambda_r^2D_{\mathrm{state}}^r.
	\label{eq:fedmix-effective-drift}
\end{equation}
Round-start alignment sets the initial mismatch to zero;
stepwise mixing controls its reappearance. Importantly,
\eqref{eq:fedmix-effective-drift} does not require local
optimizer states themselves to converge to a common state.
It allows substantial local adaptation early in training
and a progressively shared update geometry as $\lambda_r$
decreases.

\textbf{Server update and theoretical implication.}
The server recovers the normalized average local direction
from the returned model differences:
\begin{equation}
	\boldsymbol d^{r+1}
	=
	-
	\frac{1}{|\mathcal S_r|K\eta_{l,r}}
	\sum_{i\in\mathcal S_r}\Delta\boldsymbol{x}_i^r,
	\qquad
	\boldsymbol{x}^{r+1}
	=
	\boldsymbol{x}^r-\gamma_r\boldsymbol d^{r+1}.
	\label{eq:fedmix-server-update}
\end{equation}
Setting $\gamma_r=K\eta_{l,r}$ recovers direct averaging
of the returned client models. Under this setting, the
nonconvex analysis in Appendix~\ref{app:mix} shows that
the bias associated with heterogeneous preconditioning
is controlled by $\lambda_r^2\bar D_P$ together with data
heterogeneity. In particular, with bounded local-state
drift, $\eta_{l,r}=\Theta(1/(K\sqrt R))$ and
$\lambda_r=\Theta(R^{-1/2})$ for an $R$-round run,
\texttt{FedMIX-P} attains an $O(R^{-1/2})$ bound on the average
expected squared global-gradient norm. A fixed nonzero
mixing weight does not generally give this guarantee:
the appendix provides an instance with a nonstationary
fixed point.

\textbf{Instantiations and scope.}
The algorithmic template can mix gradient transformations
derived from different local optimizers. Our theoretical
guarantee applies when the \emph{applied} transformations
are linear, symmetric positive-definite preconditioners
with the stated spectral bounds. Sophia and SOAP
instantiations require their induced operators and state
aggregation to be specified explicitly. Muon's
orthogonalization is nonlinear, so a Muon-based variant
requires a separate analysis and is not covered by the
linear-preconditioner theorem.

\section{Theoretical Analysis}
\label{convergence_analysis}
\vspace{-2mm}

We analyze \texttt{FedPO} and \texttt{FedMIX-P} under
nonconvex objectives, stochastic local gradients, and
uniform partial client participation. The analysis
distinguishes the drift of local preconditioner states
from the drift of the operators actually applied to
gradients. Our guarantees concern linear, symmetric
positive-definite preconditioners; extensions to the
full nonlinear or momentum-based optimizer implementations
require separate arguments.

Let
$f(\boldsymbol{x})=\frac{1}{N}\sum_{i=1}^{N}f_i(\boldsymbol{x})$.
In every round, the server samples a size-$S$ subset
$\mathcal S_r$ uniformly without replacement. The following
assumptions state the conditions used in our analysis.

\begin{assumption}[Smoothness and lower boundedness]
	\label{smoothness}
	Each $f_i$ is differentiable and $L$-smooth:
	$
	\|\nabla f_i(\boldsymbol{x})-\nabla f_i(\boldsymbol{y})\|
	\le L\|\boldsymbol{x}-\boldsymbol{y}\|
	\text{for all }\boldsymbol{x},\boldsymbol{y}.
	$ $f(\boldsymbol{x})\ge f_*>-\infty$, and
	$\Delta:=f(\boldsymbol{x}^0)-f_*<\infty$.
\end{assumption}

\begin{assumption}[Conditional stochastic-gradient moments]
	\label{bounded_stochastic_gradient_I}
	Write
	$\boldsymbol{g}_i^{r,k}
	=\nabla f_i(\boldsymbol{x}_i^{r,k})
	+\boldsymbol{\varepsilon}_i^{r,k}$.
	Conditional on all information available before the current
	mini-batch is drawn,
$
	\mathbb E[
	\boldsymbol{\varepsilon}_i^{r,k}
	\mid\mathcal H_{r,i,k}]
	=\boldsymbol{0},
	\mathbb E[
	\|\boldsymbol{\varepsilon}_i^{r,k}\|^2
	\mid\mathcal H_{r,i,k}]
	\le\sigma^2.
$
	The sampled client set is independent of the current
	mini-batches.
\end{assumption}

\begin{assumption}[Bounded  heterogeneity]
	\label{bounded_heterogeneity}
	For every $\boldsymbol{x}$,
$
	\frac1N\sum_{i=1}^{N}
	\|\nabla f_i(\boldsymbol{x})
	-\nabla f(\boldsymbol{x})\|^2
	\le\sigma_g^2.
$
\end{assumption}

\begin{assumption}[Bounded positive-definite preconditioners]
	\label{bounded_preconditioner}
	The applied local preconditioners and the shared reference
	are symmetric linear operators satisfying, almost surely,
	\[
	mI\preceq P_i^{r,k}\preceq MI,
	\qquad
	mI\preceq P^r\preceq MI,
	\qquad 0<m\le M<\infty.
	\]
	The shared reference $P^r$ is determined before round-$r$
	client sampling and local mini-batches. For \texttt{FedPO},
	$P^r$ is an analysis reference and is not communicated.
\end{assumption}

\begin{assumption}[Bounded local preconditioner drift]
	\label{bounded_preconditioner_drift}
	There is a deterministic constant $\bar D_P<\infty$
	such that, almost surely, for every round $r$,
$
	\frac1{SK}\sum_{i\in\mathcal S_r}
	\sum_{k=1}^{K}
	\|P_i^{r,k}-P^r\|_{\mathrm{op}}^2
	\le\bar D_P.
	$
\end{assumption}

\textbf{Notation.}
Let
$
u=K\eta_l,
\Delta=f(\boldsymbol{x}^0)-f_*,
\mathcal G_R
=\frac1R\sum_{r=0}^{R-1}
\mathbb E\|\nabla f(\boldsymbol{x}^r)\|^2.
$
Define
$
\rho_S=
\begin{cases}
	\dfrac{N-S}{S(N-1)},&N>1,\\[3pt]
	0,&N=1,
\end{cases}
V=M^2\left(
\frac{\sigma^2}{SK}+\rho_S\sigma_g^2
\right),
B=\sigma_g^2+\sigma^2,
B_\chi=\sigma_g^2+\chi\sigma^2.
$
Here $\chi=0$ for predictable applied preconditioners
and $\chi=1$ when they may depend on the current mini-batch.
The notation $\lesssim$ hides only universal numerical
constants. Set
\[
D_{\mathrm{PO}}
:=
\bar D_P+\frac{(M-m)^2}{K}.
\]
The second term accounts for FedPO's retained,
potentially unaligned preconditioner at the beginning
of a round.

\begin{theorem}[Convergence bound for FedPO]
	\label{thm:main:fedpo}
	Under Assumptions~\ref{smoothness}--%
	\ref{bounded_preconditioner_drift}, consider FedPO
	with uniform size-$S$ client sampling and server model
	averaging. If $u>0$ is sufficiently small to satisfy
	the local-trajectory and descent conditions specified
	in Appendix~\ref{app:mix}, then
	\begin{equation}
		\label{eq:main:fedpo_bound}
		\mathcal G_R
		\lesssim
		\frac{\Delta}{muR}
		+\frac{Lu}{m}V
		+\frac{L^2M^4u^2}{m^2}B
		+\fcolorbox{red}{white}{$\displaystyle
			\frac{D_{\mathrm{PO}}}{m^2}B_\chi$}.
	\end{equation}
	The stepsize threshold depends on $L,m,M$, but does
	not require $D_{\mathrm{PO}}$ to be small.
\end{theorem}

\begin{theorem}[Convergence bound for \texttt{FedMIX-P}]
	\label{thm:main:fedmixp}
	Under Assumptions~\ref{smoothness}--%
	\ref{bounded_preconditioner_drift}, consider \texttt{FedMIX-P}
	with
	\[
	H_i^{r,k}
	=(1-\lambda)P^r+\lambda P_i^{r,k},
	\qquad \lambda\in[0,1],
	\]
	where $\lambda$ is fixed within an $R$-round run.
	Set the server stepsize to $\gamma=u=K\eta_l$.
	If $u>0$ and $\lambda^2\bar D_P$ are sufficiently small
	to satisfy the conditions in Appendix~\ref{app:mix},
	then
	\begin{equation}
		\label{eq:main:fedmixp_bound}
		\mathcal G_R
		\lesssim
		\frac{\Delta}{muR}
		+\frac{Lu}{m}V
		+\frac{L^2M^4u^2}{m^2}B
		+\fcolorbox{red}{white}{$\displaystyle
			\frac{\lambda^2\bar D_P}{m^2}B_\chi$}.
	\end{equation}
	In particular, for fixed problem parameters, choosing
	$u=\Theta(R^{-1/2})$ and
	$\lambda=\Theta(R^{-1/2})$ yields
	$\mathcal G_R=O(R^{-1/2})$, although
	$\bar D_P$ need only be bounded.
\end{theorem}

\textbf{Interpretation: from bounded drift to vanishing bias.}
The two bounds have the same optimization, sampling, and
local-trajectory terms. Their distinction lies in the
systematic error induced by heterogeneous preconditioning.
For FedPO, this term scales with
$D_{\mathrm{PO}}B_\chi/m^2$ and need not vanish as
$u\rightarrow0$. Reducing the learning rate alone therefore
does not remove the bias created when clients transform
heterogeneous gradients using different preconditioners.
The additional $(M-m)^2/K$ in $D_{\mathrm{PO}}$ also records
the possible mismatch of retained client states at the
beginning of a round.
\texttt{FedMIX-P} changes the quantity that governs this bias.
Round-start alignment removes the initial mismatch, while
stepwise mixing gives the exact identity
$H_i^{r,k}-P^r=\lambda(P_i^{r,k}-P^r)$.
Consequently, the preconditioner--heterogeneity term in
Theorem~\ref{thm:main:fedmixp} is proportional to
$\lambda^2\bar D_P B_\chi/m^2$. This is an
\emph{algorithmically controlled} term: the local
preconditioner states may remain different, yet the
operators applied to gradients approach the shared
geometry as $\lambda$ decreases. Thus, bounded local-state
drift suffices for a vanishing stationarity bound; the
analysis does not assume that this drift spontaneously
converges to zero.

\textbf{Rate and role of the mixing schedule.}
Taking $u=\Theta(R^{-1/2})$ balances the initial optimality
gap with stochastic and client-sampling variation. Taking
$\lambda=\Theta(R^{-1/2})$ makes the systematic
preconditioner term $O(R^{-1})$ under bounded
$\bar D_P$, yielding
$\mathcal G_R=O(R^{-1/2})$. A fixed positive $\lambda$
does not provide the same guarantee under the stated
assumptions: its drift-dependent term remains in the bound,
and Appendix~\ref{app:mix} gives a two-client example
with a nonstationary fixed point. The schedule therefore
has a specific purpose beyond interpolation between
local and global preconditioners: it removes a persistent
source of aggregation bias while allowing stronger local
adaptation earlier in training.

\textbf{Scope of the comparison.}
These results establish an advantage in the
\emph{stationarity guarantee}, not a uniform improvement
in convergence speed or wall-clock time on every problem.
The additional communication of preconditioner states
and the quality of the shared reference must be assessed
empirically. Moreover, the guarantees apply to the
linear symmetric positive-definite model specified in
the assumptions; nonlinear orthogonalization and full
optimizer-specific state updates require separate analysis.

\begin{table*}[tb]
	\centering
	\caption{\small
		Test accuracy (\%) on CIFAR-100 and Tiny-ImageNet using \textbf{ResNet-18} and \textbf{ViT-Tiny} over 300 communication rounds under Dir-0.1 and Dir-0.05.
		Experiments are conducted with 100 clients, 10\% participation, batch size 50, and $K=50$.
		Best results are highlighted in $\mathbf{bold}$, and second-best results are \underline{underlined}.
	}
	\label{tab:combined_dir01_dir005}
	\setlength{\tabcolsep}{5pt}
	\renewcommand{\arraystretch}{1.12}
	\small
	\resizebox{\textwidth}{!}{
		\begin{tabular}{lcccccccc}
			\toprule
			\multirow{3}{*}{\textbf{Method}}
			& \multicolumn{4}{c}{\textbf{ResNet-18}}
			& \multicolumn{4}{c}{\textbf{ViT-Tiny}} \\
			\cmidrule(lr){2-5} \cmidrule(lr){6-9}
			& \multicolumn{2}{c}{\textbf{CIFAR-100}}
			& \multicolumn{2}{c}{\textbf{Tiny-ImageNet}}
			& \multicolumn{2}{c}{\textbf{CIFAR-100}}
			& \multicolumn{2}{c}{\textbf{Tiny-ImageNet}} \\
			\cmidrule(lr){2-3} \cmidrule(lr){4-5} \cmidrule(lr){6-7} \cmidrule(lr){8-9}
			& \textbf{Dir-0.1} & \textbf{Dir-0.05}
			& \textbf{Dir-0.1} & \textbf{Dir-0.05}
			& \textbf{Dir-0.1} & \textbf{Dir-0.05}
			& \textbf{Dir-0.1} & \textbf{Dir-0.05} \\
			\midrule
			FedAvg
			& 60.17 & 56.75 & 47.48 & 43.80
			& 27.24 & 23.42 & 15.68 & 14.05 \\
			
			SCAFFOLD
			& 60.69 & 56.43 & 47.76 & 43.92
			& 26.86 & 23.23 & 15.70 & 14.21 \\
			
			FedCM
			& 66.61 & 62.65 & 41.16 & 36.00
			& 28.23 & 25.74 & 18.88 & 18.15 \\
			
			Local AdamW
			& 59.23 & 55.24 & 44.01 & 40.00
			& 37.57 & 36.06 & 24.31 & 21.35 \\
			
			\midrule
			Local Sophia
			& 56.65 & 50.89 & 41.23 & 36.15
			& 34.05 & 32.25 & 22.49 & 21.14 \\
			
			\rowcolor{LightMint}
			\texttt{FedMIX-P Sophia}
			& 58.56& 54.67 & 44.21 & 37.45
			& 39.67 & 32.98 & 24.23 & 22.87 \\
			
			Local Muon
			& 67.26 & 49.86 & 52.83 & 34.76
			& 44.00 & 39.68 & 30.51 & 28.25 \\
			
			\rowcolor{LightBlue}
			\texttt{FedMIX-P Muon}
			& \textbf{72.65} & \textbf{66.87} & \textbf{58.67} & \textbf{54.23}
			& 47.98 & \underline{42.45} & 32.65 & 31.56 \\
			
			Local SOAP
			& 68.44 & 58.16 & 54.42 & 50.02
			& \underline{49.41} & 41.68 & \underline{33.30} & \underline{30.36} \\
			
			\rowcolor{LightRed}
			\texttt{FedMIX-P SOAP}
			& \underline{70.34} & \underline{65.18} & \underline{56.09} & \underline{52.54}
			& \textbf{51.56} & \textbf{48.45} & \textbf{35.12} & \textbf{32.05} \\
			\bottomrule
		\end{tabular}
	}
\end{table*}

\section{Experiments}
\vspace{-2mm}
\textbf{Datasets.} We evaluate \texttt{FedMIX-P} on both vision and language tasks. (\textit{i}) For image classification, we use CIFAR-100~\cite{krizhevsky2009learning}, and Tiny-ImageNet~\cite{le2015tiny}. (\textit{ii}) For NLP tasks, we adopt C4 \cite{raffel2020exploring} dataset. To simulate data heterogeneity across clients, we follow the Dirichlet partitioning scheme~\cite{hsu2019measuring}. For Dir-$\alpha$ partitioning, smaller $\alpha$ indicates more severe data heterogeneity.\\
\textbf{Model Architectures.} We explore a variety of model types: (\textit{i}) ResNet-18~\cite{he2016deep} as a representative convolutional neural network (CNN), (\textit{ii}) Vision Transformer (ViT-Base) and ViT-Tiny~\cite{dosovitskiy2020image} for Vision Transformers, and (\textit{iii}) LLaMA \cite{touvron2023llama} for large-scale language model.\\
\textbf{Baselines.} We compare our method against  state-of-the-art FL algorithms: \texttt{FedAvg} (\texttt{Local SGD})~\cite{mcmahan2017communication}, \texttt{SCAFFOLD}~\cite{karimireddy2020scaffold}, \texttt{FedCM}~\cite{xu2021fedcm}, \texttt{Local AdamW}, \texttt{Local Sophia}, \texttt{Local Muon} and \texttt{Local SOAP}. Under our \texttt{FedMIX-P} framework, we instantiate three variants: \texttt{FedMIX-P Sophia}, \texttt{FedMIX-P Muon}, and \texttt{FedMIX-P SOAP}.
In the Appendix (Table), we compare additional FL algorithms designed to address data heterogeneity.
\\
\textbf{Hyperparameter Settings.} For \texttt{FedAvg}, \texttt{SCAFFOLD}, \texttt{FedCM}, the $lr$ is selected from 
$\{10^{-2},\ 3 \times 10^{-2},\ 5 \times 10^{-2},\ 10^{-1},\ 3 \times 10^{-1}\}$, and choose the best value $lr=10^{-1}$.
with a weight decay of $0.001$.
For \texttt{Local AdamW}, the $lr$ is selected from 
$\{10^{-4},\ 3 \times 10^{-4},\ 5 \times 10^{-4},\ 8 \times 10^{-4},\ 10^{-3}\}$, and choose the best value
$lr=3\!\times\!10^{-4}$. We set the learning rates for \texttt{Local Sophia}, \texttt{Local Muon}, and \texttt{Local SOAP} to
$3\!\times\!10^{-4}$, $3\!\times\!10^{-2}$, and $3\!\times\!10^{-3}$, respectively, $\gamma=\eta K$.
Their  weight decay is  $0.01$. We apply cosine learning rate decay, and set \texttt{FedMIX-P} to \textbf{$\boldsymbol{\lambda}\!=\!0.5$}, weight decay $0.01$. We set the learning rate of \texttt{FedMIX-P} variants to be same with \texttt{Local Sophia}, \texttt{Local Muon} and \texttt{Local SOAP}. Additional hyperparameter configurations are detailed in the \textbf{Appendix}. We release all code to ensure full reproducibility upon acceptance.  All results are averaged over 5 runs  with seeds {42, 43, 44, 45, 46}. All experiments were performed on  NVIDIA RTX 4090 GPU.

\begin{table*}[tb]
	\centering
	\caption{\small
		Comparison of different methods on vision and language tasks.
		On \textbf{ViT-Base}, results are reported under the Dir-0.1 partition over 100 communication rounds with 50 clients, 10\% participation, batch size 16, and $K=50$ local steps.
		On \textbf{LLaMA}, we report validation loss and validation perplexity (\textbf{Val PPL}) over 100 communication rounds with 20 clients, 20\% participation, batch size 16, and $K=50$ local steps.
	}
	\label{tab:merged_vit_llama}
	\setlength{\tabcolsep}{4pt}
	\renewcommand{\arraystretch}{1.12}
	\small
	\resizebox{\textwidth}{!}{
		\begin{tabular}{lcccccccccc}
			\toprule
			\multirow{3}{*}{\textbf{Method}}
			& \multicolumn{4}{c}{\textbf{ViT-Base}}
			& \multicolumn{6}{c}{\textbf{LLaMA}} \\
			\cmidrule(lr){2-5} \cmidrule(lr){6-11}
			& \multicolumn{2}{c}{\textbf{CIFAR-100}}
			& \multicolumn{2}{c}{\textbf{Tiny-ImageNet}}
			& \multicolumn{2}{c}{\textbf{LLaMA 60M}}
			& \multicolumn{2}{c}{\textbf{LLaMA 130M}}
			& \multicolumn{2}{c}{\textbf{LLaMA 350M}} \\
			\cmidrule(lr){2-3} \cmidrule(lr){4-5}
			\cmidrule(lr){6-7} \cmidrule(lr){8-9} \cmidrule(lr){10-11}
			& \textbf{Acc. $\uparrow$} & \textbf{Loss $\downarrow$}
			& \textbf{Acc. $\uparrow$} & \textbf{Loss $\downarrow$}
			& \textbf{Loss $\downarrow$} & \textbf{Val PPL $\downarrow$}
			& \textbf{Loss $\downarrow$} & \textbf{Val PPL $\downarrow$}
			& \textbf{Loss $\downarrow$} & \textbf{Val PPL $\downarrow$} \\
			\midrule
			FedAvg
			& $89.28$ & $0.332$ & $86.97$ & $0.421$
			& $5.156$ & $173.47$ & $5.258$ & $192.10$ & $5.354$ & $211.45$ \\
			
			SCAFFOLD
			& $89.35$ & $0.328$ & $86.85$ & $0.432$
			& $5.028$ & $152.63$ & $5.238$ & $188.29$ & $5.365$ & $213.79$ \\
			
			FedCM
			& $87.44$ & $0.592$ & $84.24$ & $0.708$
			& $5.231$ & $186.98$ & $5.356$ & $211.88$ & $5.426$ & $227.24$ \\
			
			Local AdamW
			& $90.10$ & $0.281$ & $86.65$ & $0.252$
			& $4.164$ & $64.33$ & $4.056$ & $57.74$ & $3.798$ & $44.61$ \\
			
			\midrule
			Local Sophia
			& $90.21$ & $0.275$ & $86.16$ & $0.188$
			& $4.012$ & $55.26$ & $3.986$ & $53.84$ & $3.521$ & $33.82$ \\
			
			\rowcolor{LightMint}
			\texttt{FedMIX-P Sophia}
			& $90.38$ & $0.262$ & $86.55$ & $0.175$
			& $3.864$ & $47.66$ & $3.752$ & $42.61$ & $3.485$ & $32.62$ \\
			
			Local Muon
			& $90.32$ & $0.189$ & $87.22$ & $0.184$
			& $3.556$ & $35.02$ & $3.456$ & $31.69$ & $3.365$ & $28.93$ \\
			
			\rowcolor{LightBlue}
			\texttt{FedMIX-P Muon}
			& \underline{$90.46$} & $0.165$ & $87.63$ & \underline{$0.171$}
			& $3.445$ & \underline{$31.34$} & \underline{$3.329$} & \underline{$27.91$} & \underline{$3.302$} & \underline{$27.17$} \\
			
			Local SOAP
			& $90.44$ & \underline{$0.151$} & \underline{$87.72$} & $0.212$
			& \underline{$3.408$} & $30.20$ & $3.345$ & $28.36$ & $3.312$ & $27.44$ \\
			
			\rowcolor{LightRed}
			\texttt{FedMIX-P SOAP}
			& $\mathbf{90.65}$ & $\mathbf{0.129}$ & $\mathbf{87.92}$ & $\mathbf{0.156}$
			& $\mathbf{3.314}$ & $\mathbf{27.49}$ & $\mathbf{3.306}$ & $\mathbf{27.28}$ & $\mathbf{3.228}$ & $\mathbf{25.23}$ \\
			\bottomrule
		\end{tabular}
	}
\end{table*}

\begin{table*}[tb]
	\centering
	\caption{\small
		Sensitivity to the mixing weight $\lambda$ for
		\texttt{FedMIX-P SOAP} on CIFAR-100 under Dir-$0.05$
		after 300 communication rounds.
		Accuracies are reported in percentages.
		The results at $\lambda=0$, $0.5$, and $1$ match
		Table~\ref{tab:soap_ablation_comm}.
		Best accuracies are highlighted in bold, and second-best
		accuracies are underlined.
		A dash indicates an unreported drift measurement.
	}
	\label{tab:beta_ablation}
	\small
	\setlength{\tabcolsep}{5pt}
	\renewcommand{\arraystretch}{1.12}
	\resizebox{\textwidth}{!}{%
		\begin{tabular}{lccccccccccc}
			\toprule
			\textbf{Model / $\lambda$}
			& \textbf{0.0} & \textbf{0.1} & \textbf{0.2}
			& \textbf{0.3} & \textbf{0.4} & \textbf{0.5}
			& \textbf{0.6} & \textbf{0.7} & \textbf{0.8}
			& \textbf{0.9} & \textbf{1.0} \\
			\midrule
			\textbf{ResNet-18}
			& $54.25$ & $62.58$ & $63.26$ & $63.22$
			& $63.49$ & $\mathbf{65.18}$ & $\mathbf{65.18}$
			& \underline{$64.26$} & $62.52$ & $58.94$ & $61.26$ \\
			\midrule
			\textbf{ViT-Tiny}
			& $38.65$ & $45.24$ & $46.11$ & $46.22$
			& $47.01$ & $\mathbf{48.45}$ & \underline{$47.05$}
			& $46.15$ & $45.74$ & $45.60$ & $45.86$ \\
			$D_P$ of ViT-Tiny $\downarrow$
			& $16.98$ & $15.78$ & $13.55$ & $11.52$
			& $9.55$ & $8.20$ & $7.05$
			& $6.19$ & $4.40$ & $2.27$ & -- \\
			\bottomrule
		\end{tabular}%
	}
\end{table*}

\subsection{Results on Convolutional Neural Networks}
\vspace{-2mm}
\paragraph{ResNet-18 on CIFAR-100 and Tiny-ImageNet.}
Table~\ref{tab:combined_dir01_dir005} reports test accuracy
after 300 communication rounds under Dir-0.1 and Dir-0.05.
\texttt{FedMIX-P} improves each corresponding local optimizer in
all eight ResNet-18 comparisons. The largest differences
occur for Muon under the more heterogeneous Dir-0.05
partition. On CIFAR-100, Local Muon declines from
$67.26\%$ at Dir-0.1 to $49.86\%$ at Dir-0.05, whereas
\texttt{FedMIX-P Muon} obtains $72.65\%$ and $66.87\%$,
respectively. Its advantage over Local Muon thus grows
from $5.39$ to $17.01$ percentage points. On
Tiny-ImageNet, the corresponding results are
$52.83\%$/$34.76\%$ for Local Muon and
$58.67\%$/$54.23\%$ for \texttt{FedMIX-P Muon},
yielding gains of $5.84$/$19.47$ points.
The improvements are not confined to Muon. On CIFAR-100,
\texttt{FedMIX-P SOAP} increases accuracy over Local SOAP
from $68.44\%$ to $70.34\%$ at Dir-0.1 and from
$58.16\%$ to $65.18\%$ at Dir-0.05. The corresponding
gains for Sophia are $1.91$ and $3.78$ points.
These comparisons are consistent with the proposed
geometry-mismatch mechanism: the benefit of controlling
locally applied preconditioners becomes more pronounced
in the settings where uncoordinated local optimization
loses the most accuracy. Accuracy alone, however, does
not isolate preconditioner drift as the cause; the drift
diagnostic and component ablations provide separate
evidence for that interpretation.

\subsection{Results on Vision Transformer}
\label{sec:vit}
\vspace{-2mm}

\textbf{Training ViT-Tiny from scratch.}
Table~\ref{tab:combined_dir01_dir005} shows that \texttt{FedMIX-P}
outperforms its corresponding local optimizer in all twelve
ViT-Tiny comparisons across two datasets and two Dirichlet partitions.
The benefit is particularly evident for SOAP on CIFAR-100:
\texttt{FedMIX-P SOAP} improves accuracy from $49.41\%$ to $51.56\%$
under Dir-$0.1$, and from $41.68\%$ to $48.45\%$ under Dir-$0.05$.
As heterogeneity increases, its advantage therefore grows from
$2.15$ to $6.77$ percentage points, while the accuracy decline between
the two partitions decreases from $7.73$ to $3.11$ points.
On Tiny-ImageNet, \texttt{FedMIX-P SOAP} achieves the best accuracy
among the evaluated methods under both partitions
($35.12\%$/$32.05\%$).
The gains also extend to Muon: \texttt{FedMIX-P Muon} improves
accuracy from $30.51\%$/$28.25\%$ to $32.65\%$/$31.56\%$.
These results support the empirical value of combining shared
geometry with local adaptation, with the SOAP results on CIFAR-100
showing reduced sensitivity to stronger data heterogeneity.

\textbf{Fine-tuning pretrained ViT-Base.}
Table~\ref{tab:merged_vit_llama} evaluates \texttt{FedMIX-P}
when training starts from a pretrained ViT-Base.
Under Dir-$0.1$, \texttt{FedMIX-P SOAP} achieves
$90.65\%$ on CIFAR-100 and $87.92\%$ on Tiny-ImageNet,
compared with $90.44\%$ and $87.72\%$ for Local SOAP.
The Sophia and Muon variants likewise obtain higher mean accuracy
than their respective local counterparts on both datasets.
Although the observed gains are modest in this setting,
they suggest that the benefits of mixing shared and local
update geometry can extend to federated fine-tuning.
Because these experiments differ from the ViT-Tiny setting
in architecture, initialization, and training protocol,
the difference in gain magnitude cannot be attributed
to pretraining alone.

\subsection{C4 Federated Pre-training with LLaMA}

Table~\ref{tab:merged_vit_llama} reports validation results after
100 communication rounds for LLaMA models with 60M, 130M, and 350M
parameters. The federated setup comprises 20 clients with 20\%
participation per round. Local adaptive optimizers substantially
improve on the first-order federated baselines; for example,
Local SOAP obtains validation losses of
$3.408/3.345/3.312$ across the three model sizes, compared with
$5.156/5.258/5.354$ for FedAvg. Preconditioner mixing delivers
further gains on top of these strong local optimizers.
\texttt{FedMIX-P SOAP} reduces the corresponding losses to
$3.314/3.306/3.228$ and achieves the lowest validation loss at
every model size. Its validation perplexities are likewise the
lowest, at $27.49/27.28/25.23$. The improvement extends beyond
SOAP: \texttt{FedMIX-P Muon} lowers Local Muon's validation loss from
$3.556/3.456/3.365$ to $3.445/3.329/3.302$.
These paired comparisons suggest that sharing a reference
preconditioner can improve federated language-model pre-training
even when the clients already use effective local optimizers.

\subsection{Ablation Study}
\vspace{-2mm}

\textbf{A1: Sensitivity to the mixing weight $\lambda$.}
The mixing weight controls how much the applied operator relies on
the evolving local preconditioner:
$H_i^{r,k}=(1-\lambda)P^r+\lambda P_i^{r,k}$.
The endpoints therefore isolate shared-only preconditioning
($\lambda=0$) and local-only preconditioning after round-wise
synchronization ($\lambda=1$). Intermediate values retain both
sources of information. We report the accuracy and measured
preconditioner-state drift across mixing weights in
Table~\ref{tab:beta_ablation}; the drift metric describes the
\emph{states}, whereas the update rule mixes the \emph{applied
	operators}. Accordingly, state drift alone does not determine
predictive accuracy.

\textbf{A2: Shared and local preconditioning.}
Table~\ref{tab:soap_ablation_comm} isolates the roles of
synchronization and within-step mixing on CIFAR-100 under Dir-$0.05$.
Shared-only preconditioning obtains $54.25\%$/$38.65\%$ accuracy
on ResNet-18/ViT-Tiny, while local-only preconditioning after
synchronization obtains $61.26\%$/$45.86\%$.
Combining the two in \texttt{FedMIX-P SOAP} increases accuracy to
$65.18\%$/$48.45\%$, exceeding either endpoint on both models.
The synchronized local-only variant also outperforms
\texttt{Local SOAP}, which retains its client-specific state across
rounds ($58.16\%$/$41.68\%$). These comparisons separately support
the value of round-wise synchronization and of mixing the shared
reference with local adaptation.

\textbf{A3: Communication-efficient state exchange.}
On ViT-Tiny, full \texttt{FedMIX-P SOAP} improves accuracy over
\texttt{Local SOAP} from $41.68\%$ to $48.45\%$, while increasing
per-round communication from $22.8$ to $45.6$\,MB.
The compressed \texttt{FedMIX-P SOAP-light} variant reaches
$47.89\%$ using $25.08$\,MB per round: it retains all but
$0.56$ percentage points of the full method's accuracy gain
at approximately $1.10\times$ the communication of Local SOAP.
Measured computation times are $5.56$\,s for Local SOAP and
$5.67$\,s for both \texttt{FedMIX-P} variants. In this setting, compressed
state exchange recovers most of the accuracy benefit with
substantially lower communication than full state exchange.

\begin{table*}[tb]
	\centering
	\caption{\small
		Ablation of shared and local preconditioning on CIFAR-100
		under Dir-$0.05$ after 300 communication rounds.
		Local SOAP and \texttt{FedMIX-P SOAP} results match
		Table~\ref{tab:combined_dir01_dir005}.
		Communication volume and computation time are measured
		per round on ViT-Tiny; communication ratios are relative
		to Local SOAP. Best accuracies are highlighted in bold.
	}
	\label{tab:soap_ablation_comm}
	\small
	\setlength{\tabcolsep}{3pt}
	\renewcommand{\arraystretch}{1.15}
	\resizebox{\textwidth}{!}{%
		\begin{tabular}{llccccc}
			\toprule
			\textbf{Method}
			& \textbf{Applied preconditioner}
			& \textbf{ResNet-18 Acc. (\%)}
			& \textbf{ViT-Tiny Acc. (\%)}
			& \textbf{Comm. (MB)}
			& \textbf{Time (s)}
			& \textbf{Comm. ratio} \\
			\midrule
			Local SOAP
			& $P_i^{r,k}$; retained local state
			& $58.16$ & $41.68$
			& $22.80$ & $5.56$ & $1.0\times$ \\
			
			Shared only
			& $P^r$ ($\lambda=0$)
			& $54.25$ & $38.65$
			& $45.60$ & $5.12$ & $2.0\times$ \\
			
			Local only after synchronization
			& $P_i^{r,k}$ ($\lambda=1$)
			& $61.26$ & $45.86$
			& $45.60$ & $5.56$ & $2.0\times$ \\
			
			\texttt{FedMIX-P SOAP}
			& $(1-\lambda)P^r+\lambda P_i^{r,k}$ ($\lambda=0.5$)
			& $\mathbf{65.18}$ & $\mathbf{48.45}$
			& $45.60$ & $5.67$ & $2.0\times$ \\
			
			\texttt{FedMIX-P SOAP-light}
			& Mixing with compressed shared state
			& $64.86$ & $47.89$
			& $25.08$ & $5.67$ & $1.1\times$ \\
			\bottomrule
		\end{tabular}%
	}
\end{table*}

\section{Conclusion}
\vspace{-2mm}
Preconditioners improve local optimization, but their effectiveness in federated training also depends on how client update geometries interact. We show that differences in client preconditioners can bias aggregation under heterogeneous data, even when gradients are evaluated at the same model. To address this problem, \texttt{FedMIX-P} mixes a shared reference with each client's adaptive preconditioner at every local step, retaining local adaptation while controlling the mismatch between applied operators. For bounded linear symmetric positive-definite preconditioners, our nonconvex analysis bounds the contribution of geometric mismatch by a term proportional to $\lambda^2\bar D_P$. A suitable horizon-dependent mixing weight therefore yields a vanishing stationarity bound without requiring local preconditioners to converge to one another. Experiments on vision and language models demonstrate improvements over the corresponding local optimizers, particularly under severe data heterogeneity. These results establish the value of preconditioner coordination for federated model training. Extending the analysis to the full nonlinear and momentum-based updates of practical optimizers remains an open problem.

\newpage

\bibliography{main}
\bibliographystyle{iclr2027_conference}

\newtheorem{mixlemma}{Lemma}[section]
\newtheorem{mixtheorem}[mixlemma]{Theorem}
\newtheorem{mixcorollary}[mixlemma]{Corollary}
\newtheorem{mixproposition}[mixlemma]{Proposition}

\newpage
\section{Convergence of FedPO and FedMIX-P}
\label{app:mix}

\paragraph{Algorithms and scope.}
Let $f(x)=N^{-1}\sum_{i=1}^N f_i(x)$ and $|\mathcal S_r|=S$.
In FedPO, client $i$ retains its preconditioner state between
participations and applies $Q_i^{r,k}=P_i^{r,k}$ at local step $k$.
The server broadcasts only $x^r$. In FedMIX-P, the server broadcasts
$(x^r,P^r)$, the client sets $P_i^{r,0}=P^r$, and applies
\begin{equation}
	Q_i^{r,k}=H_i^{r,k}:=(1-\lambda)P^r+\lambda P_i^{r,k},
	\qquad 0\le\lambda\le1.                         \label{eq:mix:operator}
\end{equation}
Both methods use the update
\begin{equation}
	x_i^{r,0}=x^r,\quad
	x_i^{r,k+1}=x_i^{r,k}-\eta Q_i^{r,k}g_i^{r,k},\quad
	x^{r+1}=\frac1S\sum_{i\in\mathcal S_r}x_i^{r,K}. \label{eq:mix:algorithm}
\end{equation}
The last identity corresponds to server stepsize $\gamma=K\eta$ in
the supplied FedMIX-P algorithm; its direction is exactly
$-(SK\eta)^{-1}\sum_i(x_i^{r,K}-x^r)$.
For FedPO, $P^r$ below is an analysis reference, not a broadcast or
an additional algorithmic operation. For FedMIX-P, it is the shared
preconditioner in the algorithm. In the two displayed algorithms, the applied $P_i^{r,k}$ is fixed
before the current mini-batch is drawn and the gradient updates
$P_i^{r,k+1}$ afterward. We also analyze an implementation that
constructs the applied operator using its current mini-batch, as
permitted by A.1--A.5.
All results concern linear symmetric positive-definite operators.
The predictable-operator results are followed by a theorem that permits
current-sample-dependent operators under A.1--A.5 alone.

\paragraph{Assumptions A.1--A.5.}
We restate the five assumptions to make the stochastic conditioning and
the drift indices precise.
\begin{enumerate}
	\item[A.1.] Each $f_i$ is $L$-smooth and $f(x)\ge f_*>-\infty$;
	$x^0$ is deterministic and $\Delta=f(x^0)-f_*<\infty$.
	\item[A.2.] Conditional on the pre-round information $\mathcal F_r$,
	$\mathcal S_r$ is a uniform size-$S$ subset, independently of
	current mini-batches. With
	$g_i^{r,k}=\nabla f_i(x_i^{r,k})+\varepsilon_i^{r,k}$ and
	$\mathcal H_{r,i,k}$ the information before its current sample,
	\begin{equation}
		\mathbb E[\varepsilon_i^{r,k}\mid\mathcal H_{r,i,k}]=0,
		\qquad
		\mathbb E[\|\varepsilon_i^{r,k}\|^2\mid\mathcal H_{r,i,k}]
		\le\sigma^2.                                  \label{eq:mix:noiseass}
	\end{equation}
	The applied operator may use the current mini-batch.
	Predictability, when invoked, means $Q_i^{r,k}$ is
	$\mathcal H_{r,i,k}$-measurable. Fresh independent
	mini-batches suffice for the calculations below.
	\item[A.3.] For every $x$,
	$N^{-1}\sum_i\|\nabla f_i(x)-\nabla f(x)\|^2\le\sigma_g^2$.
	\item[A.4.] Almost surely, $mI\preceq P_i^{r,k},P^r\preceq MI$
	for $0<m\le M<\infty$; $P^r$ is $\mathcal F_r$-measurable.
	\item[A.5.] For the clients selected in round $r$, almost surely,
	\begin{equation}
		d_{r,+}:=\frac1{SK}\sum_{i\in\mathcal S_r}
		\sum_{k=1}^{K}\|P_i^{r,k}-P^r\|^2\le\bar D_P,
		\qquad \bar D_P<\infty.                       \label{eq:mix:driftass}
	\end{equation}
\end{enumerate}
For FedPO, A.5 is a condition relative to an auxiliary predictable
reference $P^r$; it must be verified for a concrete state update. It
does not introduce preconditioner communication into FedPO.

Write
\begin{equation}
	u=K\eta,\quad G_r=\|\nabla f(x^r)\|^2,\quad
	\rho_S=\begin{cases}(N-S)/(S(N-1)),&N>1,\\0,&N=1,\end{cases}
	\quad V=M^2\left(\frac{\sigma^2}{SK}+\rho_S\sigma_g^2\right).
	\label{eq:mix:notation}
\end{equation}
The following two cases are analyzed together:
\begin{equation}
	(\theta,D)=
	\begin{cases}
		(1,\bar D_P+(M-m)^2/K),&\text{FedPO},\\
		(\lambda,\bar D_P),&\text{FedMIX-P}.
	\end{cases}                                           \label{eq:mix:cases}
\end{equation}

\begin{mixlemma}[Applied-index drift]
	\label{lem:mix:index}
	For both algorithms, almost surely,
	\begin{equation}
		d_{r,-}:=\frac1{SK}\sum_{i\in\mathcal S_r}
		\sum_{k=0}^{K-1}\|P_i^{r,k}-P^r\|^2\le D,
		\qquad
		\frac1{SK}\sum_{i\in\mathcal S_r}\sum_{k=0}^{K-1}
		\|Q_i^{r,k}-P^r\|^2\le\theta^2D.        \label{eq:mix:effective}
	\end{equation}
\end{mixlemma}
\begin{proof}
	Let $J_{r,j}=S^{-1}\sum_{i\in\mathcal S_r}
	\|P_i^{r,j}-P^r\|^2$. Cancellation of indices $1,\ldots,K-1$
	gives the exact identity
	\begin{equation}
		d_{r,-}=d_{r,+}+\frac{J_{r,0}-J_{r,K}}K
		\le\bar D_P+\frac{J_{r,0}}K.             \label{eq:mix:transfer}
	\end{equation}
	For FedMIX-P, $J_{r,0}=0$ because $P_i^{r,0}=P^r$.
	For FedPO, the bounds $mI\preceq P_i^{r,0},P^r\preceq MI$
	give $J_{r,0}\le(M-m)^2$; its retained local states are not reset.
	Finally, $Q_i^{r,k}-P^r=\theta(P_i^{r,k}-P^r)$ in
	both cases. This also establishes the result for $K=1$.
\end{proof}

\begin{mixlemma}[Local displacement]
	\label{lem:mix:trajectory}
	If $2L^2M^2u^2\le1/2$, then
	\begin{align}
		T_r&:=\mathbb E\!\left[\left.
		\frac1{SK}\sum_{i\in\mathcal S_r}\sum_{k=0}^{K-1}
		\|x_i^{r,k}-x^r\|^2\right|\mathcal F_r\right]\notag\\
		&\le4M^2u^2(G_r+\sigma_g^2)+2M^2u^2\sigma^2.
		\label{eq:mix:trajectory}
	\end{align}
\end{mixlemma}
\begin{proof}
	By A.4 and convexity of the spectral interval,
	$\|Q_i^{r,k}\|\le M$. Cauchy--Schwarz gives
	\[
	\|x_i^{r,k}-x^r\|^2
	\le\eta^2kM^2\sum_{j<k}\|g_i^{r,j}\|^2.
	\]
	Averaging over $k$ and using $k\le K$ yields
	\[
	T_r\le u^2M^2\mathbb E\!\left[\left.
	\frac1{SK}\sum_{i\in\mathcal S_r}\sum_{j=0}^{K-1}
	\|g_i^{r,j}\|^2\right|\mathcal F_r\right].
	\]
	A.2 and smoothness imply
	\[
	\mathbb E[\|g_i^{r,j}\|^2\mid\mathcal H_{r,i,j}]
	\le2\|\nabla f_i(x^r)\|^2
	+2L^2\|x_i^{r,j}-x^r\|^2+\sigma^2.
	\]
	Uniform sampling and A.3 imply
	$\mathbb E[S^{-1}\sum_{i\in\mathcal S_r}
	\|\nabla f_i(x^r)\|^2\mid\mathcal F_r]
	\le G_r+\sigma_g^2$. Hence
	\[
	T_r\le u^2M^2[2(G_r+\sigma_g^2)+\sigma^2+2L^2T_r].
	\]
	Since $2L^2M^2u^2\le1/2$, rearrangement proves the claim.
\end{proof}

\begin{mixlemma}[Direction bias and noise]
	\label{lem:mix:direction}
	Define
	\begin{equation}
		d^r=\frac1{SK}\sum_{i\in\mathcal S_r}\sum_{k=0}^{K-1}
		Q_i^{r,k}g_i^{r,k},\qquad x^{r+1}=x^r-ud^r.
		\label{eq:mix:direction}
	\end{equation}
	There exist $b^r,n^r$ such that
	$d^r=P^r\nabla f(x^r)+b^r+n^r$ and
	\begin{equation}
		\mathbb E[n^r\mid\mathcal F_r]=0,\qquad
		\mathbb E[\|n^r\|^2\mid\mathcal F_r]\le V,
		\qquad
		\mathbb E[\|b^r\|^2\mid\mathcal F_r]\le aG_r+b,
		\label{eq:mix:decompbounds}
	\end{equation}
	where, provided $2L^2M^2u^2\le1/2$,
	\begin{align}
		a&=8L^2M^4u^2+2\theta^2D,\notag\\
		b&=(8L^2M^4u^2+2\theta^2D)\sigma_g^2
		+4L^2M^4u^2\sigma^2.              \label{eq:mix:ab}
	\end{align}
\end{mixlemma}
\begin{proof}
	Set
	\begin{align*}
		A^r&=\frac1{SK}\sum_{i\in\mathcal S_r}\sum_k
		Q_i^{r,k}[\nabla f_i(x_i^{r,k})-\nabla f_i(x^r)],\\
		B^r&=\frac1{SK}\sum_{i\in\mathcal S_r}\sum_k
		(Q_i^{r,k}-P^r)\nabla f_i(x^r),\\
		C^r&=P^r\left[\frac1S\sum_{i\in\mathcal S_r}
		\nabla f_i(x^r)-\nabla f(x^r)\right],\\
		E^r&=\frac1{SK}\sum_{i\in\mathcal S_r}\sum_k
		Q_i^{r,k}\varepsilon_i^{r,k}.
	\end{align*}
	Here and below $\sum_k$ means $\sum_{k=0}^{K-1}$.
	The identity follows with $b^r=A^r+B^r$ and $n^r=C^r+E^r$.
	To see the finite-population factor, let
	$h_i=\nabla f_i(x^r)-\nabla f(x^r)$, so $\sum_i h_i=0$.
	For uniform sampling without replacement,
	$\mathbb E[\mathbf 1\{i\in\mathcal S_r\}]=S/N$ and,
	for $i\ne j$,
	$\mathbb E[\mathbf 1\{i,j\in\mathcal S_r\}]
	=S(S-1)/[N(N-1)]$. Expanding the squared sum and using
	$\sum_{i\ne j}\langle h_i,h_j\rangle=-\sum_i\|h_i\|^2$ gives
	\[
	\mathbb E\left[\left.\left\|\frac1S\sum_{i\in\mathcal S_r}h_i\right\|^2\right|\mathcal F_r\right]
	=\rho_S\frac1N\sum_i\|h_i\|^2.
	\]
	Thus
	\[
	\mathbb E[\|C^r\|^2\mid\mathcal F_r]
	\le M^2\rho_S\sigma_g^2.
	\]
	Predictability and the conditional moment bounds in A.2 imply
	that $Q_i^{r,k}\varepsilon_i^{r,k}$ are orthogonal
	martingale increments in any fixed reveal order. Specifically, for
	two distinct revealed samples $a<b$, the first increment is
	measurable immediately before $b$ is revealed, while the second
	has conditional mean zero. All cross terms therefore vanish, so
	\[
	\mathbb E[E^r\mid\mathcal F_r,\mathcal S_r]=0,
	\qquad
	\mathbb E[\|E^r\|^2\mid\mathcal F_r]
	\le\frac{M^2\sigma^2}{SK}.
	\]
	As $C^r$ is fixed conditional on $(\mathcal F_r,\mathcal S_r)$,
	its cross term with $E^r$ vanishes, proving the first two
	claims of \eqref{eq:mix:decompbounds}.
	
	Jensen's inequality, A.1, and Lemma~\ref{lem:mix:trajectory} give
	\begin{align*}
		\mathbb E[\|A^r\|^2\mid\mathcal F_r]
		&\le M^2L^2T_r\\
		&\le4L^2M^4u^2(G_r+\sigma_g^2)
		+2L^2M^4u^2\sigma^2.
	\end{align*}
	Pathwise Cauchy--Schwarz and Lemma~\ref{lem:mix:index} give
	\[
	\|B^r\|^2\le\theta^2D
	\frac1S\sum_{i\in\mathcal S_r}\|\nabla f_i(x^r)\|^2.
	\]
	Taking conditional expectation yields
	$\mathbb E[\|B^r\|^2\mid\mathcal F_r]
	\le\theta^2D(G_r+\sigma_g^2)$.
	Finally $\|A^r+B^r\|^2\le2\|A^r\|^2+2\|B^r\|^2$
	gives \eqref{eq:mix:ab}.
\end{proof}

\begin{mixtheorem}[Partial-participation stationarity bound]
	\label{thm:mix:main}
	Under A.1--A.5 and the predictable-operator implementation,
	let $a,b,V$ be defined in \eqref{eq:mix:notation} and
	\eqref{eq:mix:ab}, with $(\theta,D)$ chosen from
	\eqref{eq:mix:cases}. Suppose
	\begin{equation}
		2L^2M^2u^2\le\frac12,\qquad
		a\le\frac{m^2}{4},\qquad
		Lu(M^2+2a)\le\frac m8.                   \label{eq:mix:conditions}
	\end{equation}
	Then both algorithms obey
	\begin{equation}
		\boxed{\quad
			\frac1R\sum_{r=0}^{R-1}\mathbb E\|\nabla f(x^r)\|^2
			\le\frac{4\Delta}{muR}+\frac{2b}{m^2}
			+\frac{8Lu}{m}(b+V).\quad}        \label{eq:mix:main}
	\end{equation}
\end{mixtheorem}
\begin{proof}
	Fix $r$ and condition on $\mathcal F_r$. By A.1 and
	\eqref{eq:mix:direction},
	\begin{equation}
		f(x^{r+1})\le f(x^r)
		-u\langle\nabla f(x^r),d^r\rangle
		+\frac{Lu^2}{2}\|d^r\|^2.             \label{eq:mix:descent}
	\end{equation}
	Lemma~\ref{lem:mix:direction}, A.4, and Young's inequality show
	\begin{align}
		\mathbb E[\langle\nabla f(x^r),d^r\rangle\mid\mathcal F_r]
		&\ge mG_r-\mathbb E[\langle\nabla f(x^r),b^r\rangle
		\mid\mathcal F_r]\notag\\
		&\ge\left(\frac m2-\frac{a}{2m}\right)G_r
		-\frac{b}{2m}.              \label{eq:mix:inner}
	\end{align}
	The zero-mean term $n^r$ is absent only from the conditional
	inner product; the systematic term $b^r$ is retained.
	Also,
	\[
	\|d^r\|^2
	\le2\|P^r\nabla f(x^r)\|^2
	+4\|b^r\|^2+4\|n^r\|^2,
	\]
	and hence
	\begin{equation}
		\mathbb E[\|d^r\|^2\mid\mathcal F_r]
		\le(2M^2+4a)G_r+4b+4V.                \label{eq:mix:second}
	\end{equation}
	Combining \eqref{eq:mix:descent}--\eqref{eq:mix:second},
	\begin{align*}
		\mathbb E[f(x^{r+1})\mid\mathcal F_r]
		\le f(x^r)
		-u\left[\frac m2-\frac{a}{2m}-Lu(M^2+2a)\right]G_r
		+\frac{ub}{2m}+2Lu^2(b+V).
	\end{align*}
	Conditions \eqref{eq:mix:conditions} make the bracket at
	least $m/4$. Take total expectations, sum over rounds,
	and use $\mathbb E f(x^R)\ge f_*$ to obtain
	\eqref{eq:mix:main}.
\end{proof}

\begin{mixcorollary}[FedMIX-P converges to true stationarity]
	\label{cor:mix:rate}
	Fix all problem parameters, including $K$ and $S$, independently of $R$.
	For each horizon $R$, choose
	\begin{equation}
		u=\frac{c_u}{\sqrt R},\qquad
		\eta=\frac{c_u}{K\sqrt R},\qquad
		\gamma=u,\qquad
		\lambda=\min\!\left\{1,\frac{c_\lambda}{\sqrt R}\right\},
		\label{eq:mix:schedule}
	\end{equation}
	where $c_u>0$ and $c_\lambda\ge0$ are constants.
	For sufficiently large $R$, \eqref{eq:mix:conditions} holds and
	\begin{equation}
		\frac1R\sum_{r=0}^{R-1}\mathbb E\|\nabla f(x^r)\|^2
		=O(R^{-1/2}).                           \label{eq:mix:rate}
	\end{equation}
\end{mixcorollary}
\begin{proof}
	For FedMIX-P, $(\theta,D)=(\lambda,\bar D_P)$, so
	$a=O(R^{-1})$, $b=O(R^{-1})$, and $V$ is constant
	with respect to $R$. The three terms in \eqref{eq:mix:main}
	are $O(R^{-1/2})$, $O(R^{-1})$, and $O(R^{-1/2})$.
	All three stepsize conditions hold for sufficiently large $R$.
\end{proof}

\subsection{Current-sample-dependent preconditioners}
\label{app:mix:nonpredictable}
The five assumptions permit $P_i^{r,k}$ to use the same mini-batch
that produces $g_i^{r,k}$. In that case
$\mathbb E[P_i^{r,k}\varepsilon_i^{r,k}\mid
\mathcal H_{r,i,k}]$ need not be zero. The following theorem
keeps this correlation rather than applying a false
martingale-difference argument.

\begin{mixlemma}[Noise split without operator predictability]
	\label{lem:mix:generalnoise}
	Under A.1--A.5, without the predictable-operator condition,
	write $d^r=P^r\nabla f(x^r)+\widetilde b^r+widetilde n^r$
	for either algorithm, where
	\begin{align}
		\widetilde n^r
		&=C^r+\frac1{SK}\sum_{i\in\mathcal S_r}\sum_k
		P^r\varepsilon_i^{r,k},\notag\\
		\widetilde b^r
		&=A^r+B^r+\frac1{SK}\sum_{i\in\mathcal S_r}\sum_k
		(Q_i^{r,k}-P^r)\varepsilon_i^{r,k}.
		\label{eq:mix:generalnoise}
	\end{align}
	Then $\mathbb E[\widetilde n^r\mid\mathcal F_r]=0$ and
	$\mathbb E[\|\widetilde n^r\|^2\mid\mathcal F_r]\le V$.
	If $2L^2M^2u^2\le1/2$, then
	\begin{equation}
		\mathbb E[\|\widetilde b^r\|^2\mid\mathcal F_r]
		\le a_1G_r+b_1,                          \label{eq:mix:generalbias}
	\end{equation}
	where
	\begin{align}
		a_1&=12L^2M^4u^2+3\theta^2D,\notag\\
		b_1&=(12L^2M^4u^2+3\theta^2D)\sigma_g^2
		+(6L^2M^4u^2+3\theta^2D)\sigma^2.
		\label{eq:mix:generalab}
	\end{align}
\end{mixlemma}
\begin{proof}
	The noise in the first line of \eqref{eq:mix:generalnoise}
	is multiplied by pre-round-measurable $P^r$. Martingale
	orthogonality bounds its variance by $M^2\sigma^2/(SK)$.
	Its cross term with $C^r$ vanishes conditional on
	$(\mathcal F_r,\mathcal S_r)$.
	
	Let $W^r=(SK)^{-1}\sum_{i,k}(Q_i^{r,k}-P^r)
	\varepsilon_i^{r,k}$. Pathwise Cauchy--Schwarz and
	Lemma~\ref{lem:mix:index} give
	\[
	\|W^r\|^2\le\theta^2D
	\frac1{SK}\sum_{i,k}\|\varepsilon_i^{r,k}\|^2.
	\]
	Taking conditional expectations and using A.2 yields
	$\mathbb E[\|W^r\|^2\mid\mathcal F_r]
	\le\theta^2D\sigma^2$; no conditional independence of
	$Q_i^{r,k}$ and $\varepsilon_i^{r,k}$ is used.
	The trajectory, sampling, and drift estimates for $A^r,B^r$
	remain valid because they only require the second moment of
	$g_i^{r,k}$ and the spectral bounds. Finally,
	$\|A^r+B^r+W^r\|^2\le
	3(\|A^r\|^2+\|B^r\|^2+\|W^r\|^2)$
	gives \eqref{eq:mix:generalab}.
\end{proof}

\begin{mixtheorem}[FedMIX-P under A.1--A.5 alone]
	\label{thm:mix:general}
	For FedMIX-P let $\theta=\lambda$, $D=\bar D_P$, and
	let $a_1,b_1$ be defined by \eqref{eq:mix:generalab}.
	If
	\begin{equation}
		2L^2M^2u^2\le\tfrac12,\quad
		a_1\le\tfrac{m^2}{4},\quad
		Lu(M^2+2a_1)\le\tfrac m8,                \label{eq:mix:generalconditions}
	\end{equation}
	then
	\begin{equation}
		\frac1R\sum_{r=0}^{R-1}\mathbb E\|\nabla f(x^r)\|^2
		\le\frac{4\Delta}{muR}+\frac{2b_1}{m^2}
		+\frac{8Lu}{m}(b_1+V).      \label{eq:mix:generalbound}
	\end{equation}
	Consequently, the schedule \eqref{eq:mix:schedule} yields
	$O(R^{-1/2})$ expected stationarity without assuming
	operator predictability.
\end{mixtheorem}
\begin{proof}
	Apply A.1 to $x^{r+1}=x^r-ud^r$. The decomposition in
	Lemma~\ref{lem:mix:generalnoise} has the same three properties
	used in \eqref{eq:mix:inner}--\eqref{eq:mix:second}:
	$P^r\succeq mI$, $\widetilde n^r$ has zero conditional
	mean and second moment at most $V$, and $\widetilde b^r$
	has second moment at most $a_1G_r+b_1$. Thus
	\[
	\mathbb E[f(x^{r+1})\mid\mathcal F_r]
	\le f(x^r)-u\left[\frac m2-\frac{a_1}{2m}
	-Lu(M^2+2a_1)\right]G_r
	+\frac{ub_1}{2m}+2Lu^2(b_1+V).
	\]
	By \eqref{eq:mix:generalconditions}, the bracket is at
	least $m/4$. Sum over rounds and divide by $muR/4$ to
	obtain \eqref{eq:mix:generalbound}. Under
	\eqref{eq:mix:schedule}, $a_1,b_1=O(R^{-1})$ and
	$V=O(1)$, so the result follows.
\end{proof}

\paragraph{FedPO comparison and interpretation.}
For FedPO, a sharper argument avoids requiring the raw drift
$D_0:=\bar D_P+(M-m)^2/K$ to be small. Define
\begin{equation}
	\overline P^r:=\frac1{SK}\sum_{i\in\mathcal S_r}
	\sum_{k=0}^{K-1}P_i^{r,k}.
	\label{eq:mix:averageP}
\end{equation}
The random matrix $\overline P^r$ satisfies
$mI\preceq\overline P^r\preceq MI$ pathwise. Unlike $P^r$,
this matrix is not required to be pre-round measurable.

\begin{mixtheorem}[FedPO with arbitrary bounded preconditioner drift]
	\label{thm:mix:fedpo_general}
	Under A.1--A.5 and predictable applied operators, define
	\begin{align}
		a_0&:=8L^2M^4u^2,\notag\\
		b_0&:=(8L^2M^4u^2+2D_0)\sigma_g^2
		+4L^2M^4u^2\sigma^2.       \label{eq:mix:fedpo_ab}
	\end{align}
	If
	\begin{equation}
		2L^2M^2u^2\le\tfrac12,\qquad
		a_0\le\tfrac{m^2}{4},\qquad
		Lu(M^2+2a_0)\le\tfrac m8,               \label{eq:mix:fedpo_conditions}
	\end{equation}
	then
	\begin{equation}
		\frac1R\sum_{r=0}^{R-1}\mathbb E\|\nabla f(x^r)\|^2
		\le\frac{4\Delta}{muR}+\frac{2b_0}{m^2}
		+\frac{8Lu}{m}(b_0+V).         \label{eq:mix:fedpo_general}
	\end{equation}
	The stepsize restrictions do not depend on $D_0$.
\end{mixtheorem}
\begin{proof}
	Use the exact decomposition
	\begin{equation}
		d^r=\overline P^r\nabla f(x^r)+A^r+Z^r+C^r+E^r,
		\label{eq:mix:fedpo_decomp}
	\end{equation}
	where $A^r,C^r,E^r$ are defined in the proof of
	Lemma~\ref{lem:mix:direction} and
	\[
	Z^r=\frac1{SK}\sum_{i\in\mathcal S_r}\sum_{k=0}^{K-1}
	(P_i^{r,k}-P^r)
	[\nabla f_i(x^r)-\nabla f(x^r)].
	\]
	To verify \eqref{eq:mix:fedpo_decomp}, add and subtract the
	sample mean of the gradients at $x^r$. In particular,
	$P^r\nabla f(x^r)+B^r
	=\overline P^r\nabla f(x^r)+Z^r$.
	Lemma~\ref{lem:mix:index}, Cauchy--Schwarz, and uniform
	sampling give
	\[
	\mathbb E[\|Z^r\|^2\mid\mathcal F_r]
	\le D_0\mathbb E\!\left[\left.
	\frac1S\sum_{i\in\mathcal S_r}
	\|\nabla f_i(x^r)-\nabla f(x^r)\|^2
	\right|\mathcal F_r\right]
	\le D_0\sigma_g^2.
	\]
	The trajectory bound gives
	\[
	\mathbb E[\|A^r\|^2\mid\mathcal F_r]
	\le4L^2M^4u^2(G_r+\sigma_g^2)
	+2L^2M^4u^2\sigma^2.
	\]
	Thus the bias $A^r+Z^r$ has conditional second moment
	at most $a_0G_r+b_0$. The noise $C^r+E^r$ has conditional
	mean zero and second moment at most $V$, as before.
	Although $\overline P^r$ may depend on the current random
	samples, the inequality
	$\langle\nabla f(x^r),\overline P^r\nabla f(x^r)\rangle
	\ge mG_r$ holds pathwise. Also
	$\|\overline P^r\nabla f(x^r)\|\le M\sqrt{G_r}$
	pathwise. Therefore the proof of
	Theorem~\ref{thm:mix:main} applies with $(a,b)$ replaced
	by $(a_0,b_0)$, yielding \eqref{eq:mix:fedpo_general}.
\end{proof}

\begin{mixcorollary}[FedPO with current-sample-dependent operators]
	\label{cor:mix:fedpo_generalnoise}
	The FedPO comparison also holds under A.1--A.5 without
	operator predictability. Define
	\begin{align}
		a_{0,1}&=12L^2M^4u^2,\notag\\
		b_{0,1}&=(12L^2M^4u^2+3D_0)\sigma_g^2
		+(6L^2M^4u^2+3D_0)\sigma^2.
		\label{eq:mix:fedpo_generalab}
	\end{align}
	Under the three restrictions in
	\eqref{eq:mix:fedpo_conditions} with $a_0$ replaced by
	$a_{0,1}$,
	\begin{equation}
		\frac1R\sum_{r=0}^{R-1}\mathbb E\|\nabla f(x^r)\|^2
		\le\frac{4\Delta}{muR}+\frac{2b_{0,1}}{m^2}
		+\frac{8Lu}{m}(b_{0,1}+V).
		\label{eq:mix:fedpo_generalnoise_bound}
	\end{equation}
\end{mixcorollary}
\begin{proof}
	In \eqref{eq:mix:fedpo_decomp}, split the gradient-noise
	term as $P^r\varepsilon_i^{r,k}+
	(P_i^{r,k}-P^r)\varepsilon_i^{r,k}$.
	The first part has zero conditional mean and contributes to
	$V$ as in Lemma~\ref{lem:mix:generalnoise}. Denote the
	average of the second part by $W^r$. The pathwise drift
	bound gives
	$\mathbb E[\|W^r\|^2\mid\mathcal F_r]
	\le D_0\sigma^2$ by the same Cauchy--Schwarz calculation.
	The bias is now $A^r+Z^r+W^r$. Bounding its squared norm
	by three times the sum of squared norms gives
	$a_{0,1}G_r+b_{0,1}$. The matrix $\overline P^r$
	remains positive definite pathwise. Applying the descent
	calculation \eqref{eq:mix:descent}--\eqref{eq:mix:second}
	with these new constants completes the proof.
\end{proof}

With $u=\Theta(R^{-1/2})$, the bound
\eqref{eq:mix:fedpo_general} has a potentially nonvanishing
$O(D_0\sigma_g^2/m^2)$ term. The theorem is an upper bound:
it neither asserts that every FedPO run has a positive error
floor nor proves that FedMIX-P is faster on every instance.

\begin{mixproposition}[A fixed mixing weight can preserve a bias]
	\label{prop:mix:counterexample}
	For each fixed $\lambda>0$, there is a deterministic,
	full-participation two-client instance satisfying A.1--A.5 and
	using precisely the FedMIX-P update and preconditioner
	aggregation, for which a nonstationary fixed point exists.
\end{mixproposition}
\begin{proof}
	Take $N=S=2$, $K=2$, and
	\[
	f_1(x)=\tfrac12(x-1)^2,\quad
	f_2(x)=\tfrac12(x+1)^2,\quad
	f(x)=\tfrac12(x^2+1).
	\]
	Thus $\nabla f(x)=x$, $L=1$, $f_*=1/2$, $\sigma=0$,
	and $\sigma_g^2=1$. Fix $0<\delta<1$ and define in
	every round
	\[
	P^r=P_1^{r,0}=P_2^{r,0}=1,\qquad
	P_1^{r,1}=1+\delta,\quad P_2^{r,1}=1-\delta,\qquad
	P_1^{r,2}=P_2^{r,2}=1.
	\]
	These are deterministic, predictable, and lie in
	$[1-\delta,1+\delta]$. The server's average of the
	uploaded $P_i^{r,2}$ is $P^{r+1}=1$, exactly as required by
	the algorithm, and A.5 holds with $\bar D_P=\delta^2/2$.
	The mixed applied operators are $H_i^{r,0}=1$ and
	$H_1^{r,1}=1+\lambda\delta$,
	$H_2^{r,1}=1-\lambda\delta$.
	Write $c_1=1,c_2=-1$ and choose $0<\eta<1$.
	The first local step gives
	\[
	x_i^{r,1}-c_i=(1-\eta)(x^r-c_i).
	\]
	Therefore the normalized server direction equals
	\begin{align*}
		d^r
		&=\tfrac14\bigl[(x^r-1)+(x^r+1)
		+(1-\eta)(1+\lambda\delta)(x^r-1)
		+(1-\eta)(1-\lambda\delta)(x^r+1)\bigr]\\
		&=\frac{2-\eta}{2}x^r
		-\frac{1-\eta}{2}\lambda\delta.
	\end{align*}
	Initialize at
	\[
	x^0=x_{\lambda,\delta}
	:=\frac{(1-\eta)\lambda\delta}{2-\eta}>0.
	\]
	Then $d^r=0$ and hence $x^r=x_{\lambda,\delta}$ for
	all $r$, even though
	$\nabla f(x^r)=x_{\lambda,\delta}\ne0$.
	The same construction applies to FedPO when each client
	retains $P_i^{r,2}=1$ and uses the local state sequence
	above, corresponding to $\lambda=1$. If instead
	$\lambda=O(R^{-1/2})$, this particular fixed-point
	obstruction also vanishes at rate $O(R^{-1/2})$.
\end{proof}

\paragraph{Scope of the comparison.}
The bound for FedPO allows a fixed bias floor; it does not
assert that every FedPO run converges to a nonstationary point.
Conversely, \eqref{eq:mix:rate} is an expected stationarity
guarantee for the linear SPD model, not a guarantee for a
nonlinear Muon orthogonalization map, clipping, momentum, or an
implementation that mixes optimizer states instead of their
induced linear operators. A separate analysis is needed for those
updates. If the implementation uses $\gamma\ne K\eta$, the
server update must also be analyzed separately.

\newpage

\subsection{FedMIX-P and optimizer-specific implementations}
\label{app:fedmix_implementations}

We distinguish the exact linear-operator algorithm from its
optimizer-specific extensions. The former mixes operators applied to
the current stochastic gradient and aggregates terminal operators.
The latter specify momentum, clipping, nonlinear transformations, and
state aggregation explicitly. No historical global update direction is
used in either formulation.

\begin{algorithm}[tb]
	\small
	\caption{FedMIX-P with exact linear-operator aggregation}
	\label{alg:fedmix_exact}
	\begin{algorithmic}[1]
		\Require $x^0$, shared operator $P^0$; rounds $R$; local steps $K$;
		clients $N$; sample size $S$; stepsizes $\eta_r,\gamma_r>0$;
		local weights $\lambda_r\in[0,1]$
		\For{$r=0,\ldots,R-1$}
		\State Sample $\mathcal S_r\subseteq[N]$ uniformly without replacement,
		$|\mathcal S_r|=S$
		\State Broadcast $(x^r,P^r)$
		\For{each client $i\in\mathcal S_r$ in parallel}
		\State $x_i^{r,0}\gets x^r$; $P_i^{r,0}\gets P^r$
		\For{$k=0,\ldots,K-1$}
		\State Sample $\xi_i^{r,k}$;
		$g_i^{r,k}\gets\nabla F_i(x_i^{r,k};\xi_i^{r,k})$
		\State $d_i^{r,k}\gets
		(1-\lambda_r)P^r g_i^{r,k}+\lambda_r P_i^{r,k}g_i^{r,k}$
		\State $x_i^{r,k+1}\gets x_i^{r,k}-\eta_r d_i^{r,k}$
		\State $P_i^{r,k+1}\gets
		\Call{UpdateOperator}{P_i^{r,k},g_i^{r,k}}$
		\EndFor
		\State Upload $\delta_i^r\gets x_i^{r,K}-x^r$ and $P_i^{r,K}$
		\EndFor
		\State $d^{r+1}\gets-(SK\eta_r)^{-1}
		\sum_{i\in\mathcal S_r}\delta_i^r$
		\State $x^{r+1}\gets x^r-\gamma_r d^{r+1}$
		\State $P^{r+1}\gets S^{-1}\sum_{i\in\mathcal S_r}P_i^{r,K}$
		\EndFor
		\Statex \textit{Setting $\gamma_r=K\eta_r$ recovers direct model averaging.
			The operator update must preserve the spectral bounds used by the theorem.}
	\end{algorithmic}
\end{algorithm}

\paragraph{Practical extensions.}
For practical SOAP, Sophia, and Muon variants, the server maintains a
shared state $s_G^r$ and holds its geometry fixed within round $r$.
Participating clients initialize their complete local states from this
reference. Both branches use the current local information, rather
than a previously aggregated direction.
The complete state-exchange protocol below is an uncompressed
specification; compressed variants require an additional, explicit rule.
The local baselines use the same block routines but retain their own
states between participations and omit shared-state communication.

\begin{algorithm}[tb]
	\small
	\caption{Local baselines and practical FedMIX-P variants}
	\label{alg:fedmix_practical}
	\begin{algorithmic}[1]
		\Require $x^0$; rounds $R$; local steps $K$; clients $N$; sample size $S$;
		stepsizes $\eta_r,\gamma_r>0$; weights $\lambda_r\in[0,1]$;
		optimizer $\mathsf O\in\{\mathrm{SOAP},\mathrm{Sophia},\mathrm{Muon}\}$;
		mode $\mathsf A\in\{\mathrm{Local},\mathrm{FedMIX\mbox{-}P}\}$;
		decoupled weight decay $w\geq0$
		\State Initialize shared state $s_G^0$ and all client states $s_i$
		using the initialization specified below
		\State Initialize persistent local-baseline counters $\tau_i\gets0$
		\For{$r=0,\ldots,R-1$}
		\State Sample $\mathcal S_r$ uniformly without replacement,
		$|\mathcal S_r|=S$; broadcast $x^r$
		\If{$\mathsf A=\mathrm{FedMIX\mbox{-}P}$}
		\State Also broadcast $s_G^r$; set $\ell\gets\lambda_r$
		\Else
		\State Set $\ell\gets1$
		\EndIf
		\For{each client $i\in\mathcal S_r$ in parallel}
		\State $x_i^{r,0}\gets x^r$
		\If{$\mathsf A=\mathrm{FedMIX\mbox{-}P}$}
		\State $s_i\gets\Call{Copy}{s_G^r}$;
		$s_{\rm ref}\gets\Call{Copy}{s_G^r}$;
		$\tau_i\gets rK$
		\Else
		\State $s_{\rm ref}\gets\Call{Copy}{s_i}$
		\Comment{Reference branch is unused}
		\EndIf
		\For{$k=0,\ldots,K-1$}
		\State Sample $\xi_i^{r,k}$;
		$g_i^{r,k}\gets\nabla F_i(x_i^{r,k};\xi_i^{r,k})$;
		$\tau_i\gets\tau_i+1$
		\If{$\mathsf O=\mathrm{SOAP}$}
		\State Apply \textsc{SOAPMixStep} to each matrix block
		and assemble $d_i^{r,k}$
		\ElsIf{$\mathsf O=\mathrm{Sophia}$}
		\State $d_i^{r,k}\gets
		\Call{SophiaMixStep}{g_i^{r,k},s_i,s_{\rm ref},\ell,
			x_i^{r,k},\tau_i}$
		\Else
		\State Apply \textsc{MuonMixStep} to each matrix block
		and assemble $d_i^{r,k}$
		\EndIf
		\State $x_i^{r,k+1}\gets(1-\eta_r w)x_i^{r,k}
		-\eta_r d_i^{r,k}$
		\EndFor
		\State Upload $\delta_i^r\gets x_i^{r,K}-x^r$
		\If{$\mathsf A=\mathrm{FedMIX\mbox{-}P}$}
		\State Also upload the terminal state $s_i$
		\Else
		\State Retain $s_i$ and $\tau_i$ for future participation
		\EndIf
		\EndFor
		\State $x^{r+1}\gets x^r+
		\frac{\gamma_r}{SK\eta_r}\sum_{i\in\mathcal S_r}\delta_i^r$
		\If{$\mathsf A=\mathrm{FedMIX\mbox{-}P}$}
		\State $s_G^{r+1}\gets
		\Call{AggregateState}{\mathsf O,\{s_i:i\in\mathcal S_r\}}$
		\EndIf
		\EndFor
		\Statex \textit{Reference states remain immutable within each round.
			The routines below update local states in place.}
	\end{algorithmic}
\end{algorithm}

\subsubsection{SOAP: mixing induced linear geometry}
\label{app:fedmix_soap}

For a matrix block, store
$s=(C_L,C_R,Q_L,Q_R,M,V)$, where $M$ is an ambient-coordinate
first moment and $V$ is a second moment in the coordinates defined by
$Q_L,Q_R$. Storing $M$ in ambient coordinates is equivalent to exact
first-moment transport under a change of basis.
Initialize $C_L=C_R=0$, $Q_L=I_m$, $Q_R=I_n$, and $M=V=0$.
Define the linear map
\[
\mathcal P_s(Z)=Q_L\left[
(Q_L^\top ZQ_R)\oslash(\sqrt V+\epsilon)
\right]Q_R^\top.
\]
The practical SOAP update mixes $\mathcal P_{s_{\rm ref}}$
and $\mathcal P_{s_i}$ on the same current momentum matrix.
Its geometry states are updated after the model direction is formed.

\begin{algorithm}[tb]
	\small
	\caption{FedMIX-P SOAP block step}
	\label{alg:fedmix_soap_block}
	\begin{algorithmic}[1]
		\Require $\beta_1,\beta_2\in[0,1)$; $\epsilon>0$;
		basis-refresh interval $f_Q\in\mathbb N_+$
		\Function{SOAPMixStep}{$G,s_i,s_{\rm ref},\ell,\tau$}
		\State Unpack $s_i=(C_L,C_R,Q_L,Q_R,M,V)$
		\State $M^+\gets\beta_1M+(1-\beta_1)G$
		\State $D\gets(1-\ell)\mathcal P_{s_{\rm ref}}(M^+)
		+\ell\mathcal P_{s_i}(M^+)$
		\State $C_L\gets\beta_2C_L+(1-\beta_2)GG^\top$
		\State $C_R\gets\beta_2C_R+(1-\beta_2)G^\top G$
		\If{$\tau\bmod f_Q=0$}
		\State $Q_L^+\gets\Call{Eigenbasis}{C_L}$;
		$Q_R^+\gets\Call{Eigenbasis}{C_R}$
		\State $A\gets(Q_L^+)^\top Q_L$;
		$B\gets Q_R^\top Q_R^+$
		\State $V\gets(A\odot A)V(B\odot B)$
		\Comment{Diagonal second-moment transport}
		\State $Q_L\gets Q_L^+$; $Q_R\gets Q_R^+$
		\EndIf
		\State $\widehat G\gets Q_L^\top GQ_R$
		\State $V\gets\beta_2V+(1-\beta_2)
		(\widehat G\odot\widehat G)$
		\State Store $(C_L,C_R,Q_L,Q_R,M^+,V)$ in $s_i$
		\State \Return $D$
		\EndFunction
	\end{algorithmic}
\end{algorithm}

\paragraph{Aggregation in a common coordinate system.}
For each block, average the terminal statistics and ambient moments:
\[
C_{L,G}^+=\frac1S\sum_iC_{L,i},\qquad
C_{R,G}^+=\frac1S\sum_iC_{R,i},\qquad
M_G^+=\frac1S\sum_iM_i,
\]
where sums run over $i\in\mathcal S_r$.
Set $Q_{L,G}^+=\operatorname{Eigenbasis}(C_{L,G}^+)$ and
$Q_{R,G}^+=\operatorname{Eigenbasis}(C_{R,G}^+)$, and define
\[
A_i=(Q_{L,G}^+)^\top Q_{L,i},\qquad
B_i=Q_{R,i}^\top Q_{R,G}^+,
\qquad
V_G^+=\frac1S\sum_i(A_i\odot A_i)V_i(B_i\odot B_i).
\]
The next reference is
$s_G^+=(C_{L,G}^+,C_{R,G}^+,Q_{L,G}^+,Q_{R,G}^+,M_G^+,V_G^+)$.
The second-moment transport discards untracked cross-coordinate
second moments. Averaging $V_i$ without transporting their coordinates
would not be a valid common-coordinate aggregation.

\subsubsection{Sophia: mixing inverse curvature before clipping}
\label{app:fedmix_sophia}

We specify a Sophia-H extension with state $s=(m,h)$,
initialized by $m=h=0$.
Define the elementwise inverse-curvature scale
\[
p(h)=\boldsymbol 1\oslash\max\{\rho h,\epsilon\boldsymbol 1\},
\qquad \rho>0.
\]
The mixed scale is $(1-\ell)p(h_{\rm ref})+\ell p(h_i)$.
Clipping is applied once, after the two preconditioned contributions
are combined. This specifies an inverse-scale mixture, rather than
an inverse of an averaged curvature estimate.

\begin{algorithm}[tb]
	\small
	\caption{Local Sophia-H and FedMIX-P Sophia-H step}
	\label{alg:fedmix_sophia_step}
	\begin{algorithmic}[1]
		\Require $\beta_1,\beta_2\in[0,1)$; $\rho,\epsilon>0$;
		curvature-refresh interval $f_H\in\mathbb N_+$
		\Function{SophiaMixStep}{$g,s_i,s_{\rm ref},\ell,x,\tau$}
		\State Unpack $s_i=(m_i,h_i)$ and $s_{\rm ref}=(m_{
			m ref},h_{
			m ref})$
		\State $m_i^+\gets\beta_1m_i+(1-\beta_1)g$
		\State $p_i\gets\boldsymbol1\oslash
		\max\{\rho h_i,\epsilon\boldsymbol1\}$
		\State $p_{
			m ref}\gets\boldsymbol1\oslash
		\max\{\rho h_{
			m ref},\epsilon\boldsymbol1\}$
		\State $z\gets[(1-\ell)p_{
			m ref}+\ell p_i]\odot m_i^+$
		\State $d\gets\operatorname{clip}(z,-1,1)$
		\Comment{Coordinatewise clipping after mixing}
		\If{$(\tau-1)\bmod f_H=0$}
		\State Sample a curvature mini-batch $\zeta$ independently
		of the current gradient mini-batch
		\State Draw $u\sim\mathcal N(0,I)$ independently
		\State $v\gets\nabla_x^2F_i(x;\zeta)u$
		\Comment{Hessian--vector product}
		\State $\widehat h\gets u\odot v$
		\State $h_i^+\gets\beta_2h_i+(1-\beta_2)\widehat h$
		\Else
		\State $h_i^+\gets h_i$
		\EndIf
		\State Store $(m_i^+,h_i^+)$ in $s_i$
		\State \Return $d$
		\EndFunction
	\end{algorithmic}
\end{algorithm}

For Sophia, \textsc{AggregateState} returns
\[
m_G^+=\frac1S\sum_{i\in\mathcal S_r}m_i,\qquad
h_G^+=\frac1S\sum_{i\in\mathcal S_r}h_i.
\]
Local Sophia-H uses $\ell=1$ and retained local states.
The refreshed curvature is used by subsequent steps.
Sophia-G can replace the Hutchinson estimator with a
Gauss--Newton--Bartlett estimator; that replacement must specify
the predictive-label sampling and loss/batch normalization.
The empirical squared gradient on observed labels is not a substitute
for a specified Sophia-G curvature estimator.

\subsubsection{Muon: mixing current momentum-conditioned maps}
\label{app:fedmix_muon}

For each matrix block, initialize the momentum state $M=0$.
The shared branch incorporates the current gradient into the frozen
reference momentum, while the local branch incorporates the same
gradient into the evolving local momentum. Thus the shared branch
does not reuse an old orthogonalized update direction.

\begin{algorithm}[tb]
	\small
	\caption{Local Muon and nonlinear FedMIX-P Muon block step}
	\label{alg:fedmix_muon_step}
	\begin{algorithmic}[1]
		\Require $\mu\in[0,1)$; $\epsilon>0$; iterations $s\in\mathbb N_+$;
		block scaling $a(m,n)>0$
		\Function{MuonMixStep}{$G,M_i,M_{
				m ref},\ell$}
		\State $(m,n)\gets\operatorname{shape}(G)$
		\State $M_i^+\gets\mu M_i+(1-\mu)G$
		\State $Z_{
			m ref}\gets\mu M_{
			m ref}+(1-\mu)G$
		\State $U_i\gets\Call{NewtonSchulz}{M_i^+,s,\epsilon}$
		\If{$\ell<1$}
		\State $U_{
			m ref}\gets
		\Call{NewtonSchulz}{Z_{
				m ref},s,\epsilon}$
		\State $D\gets a(m,n)[(1-\ell)U_{
			m ref}+\ell U_i]$
		\Else
		\State $D\gets a(m,n)U_i$
		\EndIf
		\State Store $M_i^+$ as the local momentum state
		\State \Return $D$
		\EndFunction
		\Statex
		\Function{NewtonSchulz}{$G,s,\epsilon$}
		\State $(m,n)\gets\operatorname{shape}(G)$
		\State $X\gets G/(\|G\|_F+\epsilon)$
		\If{$m>n$}
		\State $X\gets X^\top$
		\EndIf
		\For{$j=1,\ldots,s$}
		\State $A\gets XX^\top$
		\State $X\gets\tfrac12(3I-A)X$
		\EndFor
		\If{$m>n$}
		\State $X\gets X^\top$
		\EndIf
		\State \Return $X$
		\EndFunction
	\end{algorithmic}
\end{algorithm}

For Muon, \textsc{AggregateState} returns
$M_G^+=S^{-1}\sum_{i\in\mathcal S_r}M_i$ for each block.
The displayed Newton--Schulz routine is the cubic polar-iteration
variant. Other polynomials, Nesterov inputs, and block scalings
define different implementations and must be stated explicitly.
Rank-deficient inputs need not yield full-rank orthogonal outputs.

\paragraph{Relation to the linear-operator theorem.}
Algorithm~\ref{alg:fedmix_exact} is the algorithm covered by the
linear-SPD analysis, subject to its stated assumptions and stepsizes.
Algorithms~\ref{alg:fedmix_practical}--\ref{alg:fedmix_muon_step}
specify proposed practical extensions, not implementations validated
by the existing experimental tables.
SOAP and Sophia mix induced linear geometry before their additional
momentum or clipping operations. Muon mixes nonlinear,
momentum-conditioned maps.
In the practical state-aggregation rules, generally
\[
\mathcal P_{\operatorname{AggregateState}(\{s_i\})}
\ne\frac1S\sum_i\mathcal P_{s_i}.
\]
Therefore these rules do not reproduce exact terminal-operator
aggregation. Their convergence requires additional analysis, even
when their induced geometries are positive definite.
For a theorem-matched specialization, use current gradients instead
of momentum, remove clipping and decoupled weight decay, enforce
uniform spectral bounds, and aggregate the terminal applied operators
exactly as in Algorithm~\ref{alg:fedmix_exact}.

\paragraph{Reproducibility details.}
Report all EMA coefficients, damping values, refresh intervals,
curvature-estimation batches, block scalings, state-exchange precision,
and parameter-group rules. SOAP uses uncorrected moments here;
adding bias correction changes this specification.
Non-matrix parameters in a Muon implementation require a specified
auxiliary optimizer. A shared-only setting means shared geometry;
SOAP and Sophia may still maintain evolving local momentum.
Measured communication and computation costs must be recomputed
for these full-state and, for Muon, potentially two-transform routines.

\section{Related Work}
\label{sec:related_work}

\paragraph{Communication-efficient federated optimization.}
Communication-efficient federated learning commonly allows clients to perform
multiple local updates between aggregation rounds. FedBCGD reduces the amount
of information transmitted in each round by updating and communicating
parameter blocks, and further incorporates drift control and variance
reduction in its accelerated variant~\citep{liu2024fedbcgd}. This approach
addresses the size of each message. In contrast, \texttt{CTP-FL} communicates
one model-sized vector per participating client and changes where local
gradients are evaluated: clients query the same server-defined predictive
trajectory, so their updates can be averaged without first combining models
that have followed different local trajectories.

\paragraph{Alignment under heterogeneous data.}
Several recent methods study different forms of local--global misalignment.
FedSWA and FedMoSWA use stochastic weight averaging and momentum-based control
to improve generalization under highly heterogeneous data~\citep{liu2025fedswa}.
FedNSAM examines the mismatch between local and global flatness and uses a
global Nesterov direction to improve their consistency~\citep{liu2025fednsam}.
These methods primarily target the properties of the resulting solution,
including flatness and generalization. Our focus is the geometry of gradient
evaluation during a communication round: when all clients evaluate at common
points, averaging their gradients estimates the gradient of the global
objective at those points, irrespective of how different the individual
client gradients are.

\paragraph{Federated adaptive and structured optimizers.}
FedAdamW combines local correction, decoupled weight decay, and aggregation
of second-moment estimates for federated large-model
training~\citep{liu2026fedadamw}. FedMuon exploits matrix orthogonalization
and local--global alignment to improve federated optimization of
matrix-structured parameters~\citep{liu2025fedmuon}. FedPAC identifies
preconditioner drift as a source of instability when local second-order
optimizers induce incompatible client geometries, and proposes
preconditioner alignment and update correction~\citep{liu2026fedpac}.
Unlike these optimizer-specific mechanisms, \texttt{CTP-FL} applies to
stochastic-gradient evaluations without transmitting moments or
preconditioners. Its common trajectory aligns the \emph{locations} of
gradient evaluation rather than optimizer states.

\paragraph{Global flatness and privacy.}
DP-FedPGN encourages globally flat solutions in client-level differentially
private federated learning through a global gradient-norm
penalty~\citep{liu2025dpfedpgn}. Its objective and privacy accounting are
different from ours. We cite it because it likewise illustrates that a
quantity defined by the global objective need not be faithfully represented
by independently optimized local objectives.
\end{document}